\documentclass{article}

\usepackage{microtype}
\usepackage{graphicx}
\usepackage{subcaption}
\usepackage{booktabs} 
\usepackage{amssymb}
\usepackage{amsmath}
\usepackage{multirow}
\usepackage{hyperref}
\usepackage{enumitem}
\usepackage{amsthm}
\newtheorem{theorem}{Theorem}[section]
\newtheorem*{theorem*}{Theorem}
\newtheorem{definition}{Definition}[section]

\usepackage[accepted]{icml2020new}

\usepackage[dvipsnames]{xcolor}

\definecolor{dark-blue}{rgb}{0.15,0.15,0.4}
\definecolor{medium-blue}{rgb}{0,0,0.5}
\hypersetup{
  colorlinks, linkcolor={dark-blue},
  citecolor={dark-blue}, urlcolor={medium-blue}
}

\icmltitlerunning{Effective Dimensionality Revisited}

\begin{document}

\twocolumn[
\icmltitle{Rethinking Parameter Counting in Deep Models: \\ Effective Dimensionality Revisited}

\icmlsetsymbol{equal}{*}

\begin{icmlauthorlist}
\icmlauthor{Wesley J. Maddox$^*$}{} \quad
\icmlauthor{Gregory Benton$^*$}{} \quad
\icmlauthor{Andrew Gordon Wilson}{} \\
\hspace{-0.97cm}
New York University
\end{icmlauthorlist}

\vskip 0.3in
]

\printAffiliationsAndNotice{\icmlEqualContribution}

\begin{abstract}
Neural networks appear to have mysterious generalization properties when using parameter counting as a proxy for complexity.
Indeed, neural networks often have many more parameters than there are data points, yet still provide good generalization performance.
Moreover, when we measure generalization as a function of parameters, we see \emph{double descent} behaviour, where the test error decreases, increases, and then again decreases.
We show that many of these properties become understandable when viewed through the lens of \emph{effective dimensionality}, which measures the dimensionality of the parameter space determined by the data.
We relate effective dimensionality to posterior contraction in Bayesian deep learning, model selection, width-depth tradeoffs, double descent, and functional diversity in loss surfaces,  leading to a richer understanding of the interplay between parameters and functions in deep models. We also show that effective dimensionality compares favourably to alternative norm- and flatness- based generalization measures.
\end{abstract}

\section{Introduction}

Parameter counting is often used as a proxy for model complexity to reason about generalization \citep[e.g.,][]{zhang2016understanding, shazeer2017outrageously,belkin2019reconciling}, but it can be a poor description of both model flexibility and inductive biases. One can easily construct degenerate cases, such as predictions being generated by a sum of parameters, where the number of parameters is divorced from the statistical properties of the model.
When reasoning about generalization,  \emph{overparametrization} is besides the point: what matters is how the parameters combine with the functional form of the model.

Indeed, the practical success of convolutional neural networks (CNNs) for image recognition tasks is almost entirely about the inductive biases of convolutional filters, depth, and sparsity, for extracting local similarities and hierarchical representations, rather than flexibility
\citep{lecun_lenet_1998,szegedy2015going}.
Convolutional neural networks have far fewer parameters than fully connected networks, yet can provide much better generalization. Moreover, width can provide flexibility, but it is \emph{depth} that has made neural networks distinctive in their generalization abilities.
\begin{figure}[!t]
    \centering
    \includegraphics[width=\linewidth]{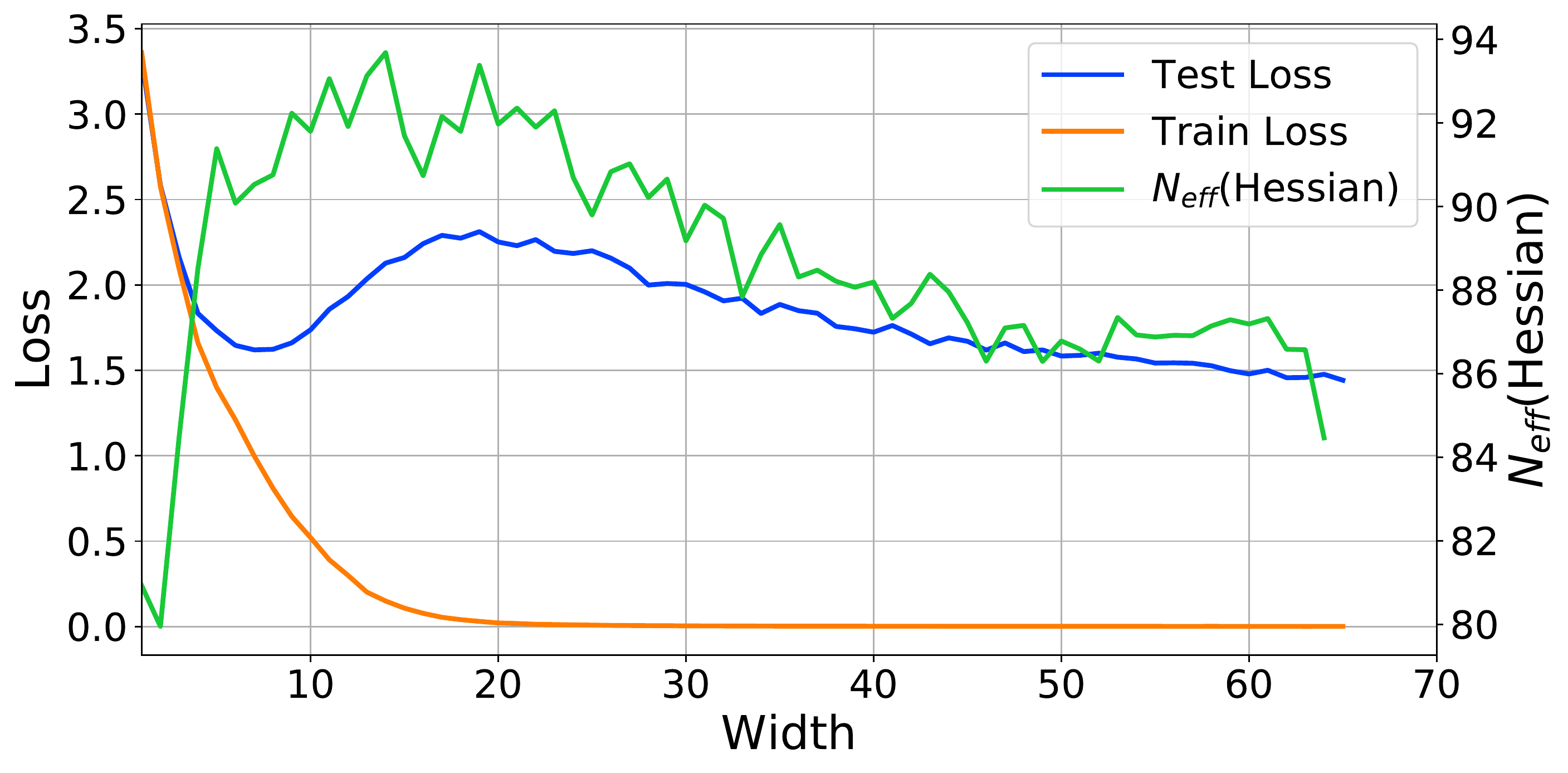}
    \caption{\textbf{A resolution of double descent}. We replicate the \emph{double descent} behaviour of deep neural networks using a ResNet18 \citep{he2016identity} on CIFAR-$100,$ where train loss decreases to zero with sufficiently wide model while test loss decreases, then increases, and then decreases again. Unlike model width, the \emph{effective dimensionality} computed from the eigenspectrum of the Hessian of the loss on \emph{training data alone} follows the test loss in the overparameterized regime, acting as a much better proxy for generalization than naive parameter counting.}
    \label{fig:nn-double-descent-intro}
\end{figure}

\begin{figure*}[!t]
    \centering
    \includegraphics[width=\linewidth]{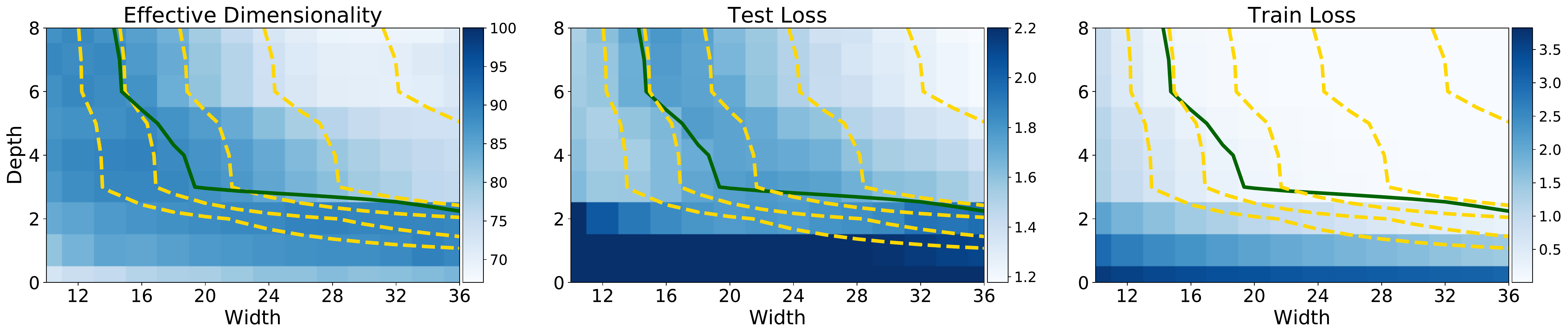}
    \caption{\textbf{Left:} Effective dimensionality as a function of model width and depth for a CNN on CIFAR-$100$. \textbf{Center:} Test loss as a function of model width and depth. \textbf{Right:} Train loss as a function of model width and depth. Yellow level curves represent equal parameter counts ($1e5$, $2e5$, $4e5$, $1.6e6$). The green curve separates models with near-zero training loss. Effective dimensionality serves as a good proxy for generalization for models with low train loss. We see wide but shallow models overfit, providing low train loss, but high test loss and high effective dimensionality. For models with the same train loss, lower effective dimensionality can be viewed as a better compression of the data at the same fidelity. Thus depth provides a mechanism for compression, which leads to better generalization.
    }
    \label{fig:width_depth_exp}
\end{figure*}

In this paper, we argue that we should move beyond simple parameter counting, and show how the generalization properties of neural networks become interpretable through the lens of \emph{effective dimensionality} \citep{mackay1992bayesian}.
Effective dimensionality was originally proposed to measure how many directions in parameter space had been determined in a Bayesian neural network, by computing the eigenspectrum of the Hessian on the training loss (Eq.~\eqref{eqn: eff-dim}, Section \ref{sec: background}).
We provide explicit connections between effective dimensionality,  posterior contraction, model selection, loss surfaces, and generalization behaviour in modern deep learning.

Consider Figure \ref{fig:nn-double-descent-intro}, where we see that once a model has achieved low training loss, the effective
dimensionality, computed from training data alone, closely tracks the mysterious double descent behaviour for neural networks. Indeed, models with increasing width actually have lower effective dimensionality, and better generalization.
Such models can be viewed as providing a better compression of the data, despite having more parameters. 
Alternatively, consider Figure \ref{fig:width_depth_exp}, where we see that width and depth determine effective dimensionality in different ways, though both are related to numbers of parameters.
Remarkably, for models with low training loss (above the green partition), the effective dimensionality closely tracks generalization performance for each combination. We also see that wide but shallow models overfit, while depth helps provide lower effective dimensionality, leading to a better compression of the data.

Specifically, our contributions are as follows:

\begin{itemize}

    \item In Section \ref{sec: post-con-bnn}, we demonstrate that the effective dimensionality of a Bayesian neural network is inversely proportional to the variance of the posterior distribution. As the effective dimensionality increases, so does the dimensionality of parameter space in which the posterior variance has contracted. 

    \item In Section \ref{sec: func-homog-blr}, we prove analytically for linear and generalized linear models with isotropic priors, $k$ parameters, and $n < k$ data observations,  there exist $k-n$ orthogonal directions
    corresponding to eigenvectors of the Hessian of the loss with smallest eigenvalues, where perturbations in parameters lead to almost no difference in the predictions of the corresponding functions on training data. That is, $f(x;\theta) \approx f(x;\widetilde{\theta})$, for training inputs $x$ and perturbed parameters $\widetilde{\theta}$.

    \item In Section \ref{sec: loss-surfaces}, we demonstrate experimentally that there exist \emph{degenerate} directions in parameter space for overparameterized deep neural networks, determined from the eigenvectors of the Hessian, which lead to function-space homogeneity for both training \emph{and} testing inputs. The presence of a high dimensional parameter space in which functional predictions are unchanged suggests that subspace and ensembling methods could be improved through the avoidance of expensive computations within degenerate parameter regimes. This finding also leads to an interpretation of effective dimensionality as \emph{model compression}, since the undetermined directions do not contain additional functional information.

    \item In Section \ref{sec: double-descent}, we show that effective dimensionality provides a compelling mechanism for model selection, resolving generalization behaviour in deep learning that appears mysterious from the perspective of simple parameter counting, such as double-descent and width-depth tradeoffs. Given two models with the same training error, the one with lower effective dimensionality, but not necessarily fewer parameters, should be preferred. 
    
    \item In Section \ref{sec: generalization}, we compare effective dimensionality to the path-norm and PAC-Bayesian flatness generalization measures, showing a greater ability to both track and explain generalization performance when comparing between large modern neural networks.
\end{itemize}

We make code available at \url{https://github.com/g-benton/hessian-eff-dim}.

\section{Posterior Contraction, Effective Dimensionality, and the Hessian}
\label{sec: background}

We consider a model, typically a neural network, $f(x; \theta)$, with inputs $x$ and parameters $\theta \in \mathbb{R}^k$. 
We define the Hessian as the $k \times k$ matrix of second derivatives of the loss,
$\mathcal{H}_\theta = -\nabla \nabla_\theta \mathcal{L}(\theta, \mathcal{D}).$ Often the loss used to train a model by optimization
is taken to be the negative log posterior $\mathcal{L} = - \log p(\theta | \mathcal{D})$.

To begin, we describe posterior contraction, effective dimensionality, and connections to the Hessian.

\subsection{Posterior Contraction}
\label{sec: postcon}

\begin{definition}
We define \textit{posterior contraction} of a set of parameters, $\theta$, as the difference in the  trace of prior and posterior covariance.
\begin{equation}
\Delta_{post}(\theta) = tr(Cov_{p(\theta)}(\theta)) - tr(Cov_{p(\theta | \mathcal{D})}(\theta)),
\label{eqn: postc}
\end{equation}
where $p(\theta)$ is the prior distribution and $p(\theta | \mathcal{D})$ is the posterior distribution given data, $\mathcal{D}$.
\end{definition}
With increases in data the posterior distribution of parameters becomes increasingly concentrated around a single value 
\citep[e.g.,][Chapter 10]{vaart_asymptotic_1998}.
Therefore Eq.~\eqref{eqn: postc} serves to measure the increase in certainty about the parameters under the posterior as compared to the prior.

\subsection{Parameter Space and Function Space}

When combined with the functional form of a model, a distribution over parameters $p(\theta)$ induces a distribution over functions $p(f(x;\theta))$. The parameters
are of little direct interest --- what matters for generalization is the distribution over functions.
Figure \ref{fig: posterior-contraction} provides both \textit{parameter-} and \textit{function-space} viewpoints. As parameter distributions concentrate around specific values, we expect to generate less diverse functions. 

We show in Appendix \ref{app:linear_eff} 
that the posterior contraction for Bayesian linear regression, $y \sim \mathcal{N}(f=\Phi^\top \beta, \sigma^2 I)$, with isotropic Gaussian prior, $\beta \sim \mathcal{N}(0, \alpha^2I_N),$ is given by
\begin{align}
    \Delta_{post}(\theta) = \alpha^2 \sum_{i=1}^N \frac{\lambda_i}{\lambda_i + \alpha^{-2}},
    \label{eq:delta_posterior}
\end{align}
where $\lambda_i$ are the eigenvalues of $\Phi^\top \Phi.$
This quantity is distinct from the posterior contraction in function space (also shown in Appendix \ref{app:linear_eff}).
We refer to the summation in Eq. ~\eqref{eq:delta_posterior} as the \textit{effective dimensionality} of $\Phi^\top \Phi$.

\subsection{Effective Dimensionality}\label{sec: eff-dim}

\begin{definition}
The effective dimensionality of a symmetric matrix $A \in \mathbb{R}^{k \times k}$ is defined as
\begin{equation}
    N_{eff}(A, z) = \sum_{i=1}^{k}\frac{\lambda_i}{\lambda_i + z} \,,
\label{eqn: eff-dim}
\end{equation}
in which $\lambda_i$ are the eigenvalues of $A$ and $z>0$ is a regularization constant \cite{mackay1992bayesian}.
\end{definition}
Typically as neural networks are trained we observe a gap in the eigenspectrum of the Hessian of the loss \citep{sagun2016eigenvalues};
a small number of eigenvalues become large while the rest take on values near $0$.
In this definition of effective dimensionality, eigenvalues much larger than $z$ contribute a value of approximately $1$ to the summation, and eigenvalues much smaller than $z$ contribute a value of approximately $0.$

\subsection{The Hessian and the Posterior Distribution}
\label{sec: blrexample}

We provide a simple example involving posterior contraction, effective dimensionality, and their connections to the Hessian. Figure \ref{fig: posterior-contraction} shows the prior and posterior distribution for a Bayesian linear regression model with a single parameter, with predictions generated by parameters drawn from these distributions.
As expected from Sections \ref{sec: postcon} and \ref{sec: eff-dim}, we see that the variance of the posterior distribution is significantly reduced from that of the prior --- what we refer to here as \textit{posterior contraction}.

We can see from Figure \ref{fig: posterior-contraction} that the arrival of data increases the curvature of the loss (negative log posterior) at the optimum.
This increase in curvature of the loss that accompanies certainty about the parameters leads to an increase in the eigenvalues of the Hessian of the loss in the multivariate case. 
Thus, growth in eigenvalues of the Hessian of the loss corresponds to increased certainty about parameters, leading to the use of the effective dimensionality of the Hessian of the loss as a proxy for the number of parameters that have been determined.\footnote{Empirically described in Appendix \ref{app:hess_training}.}

We often desire models that are both consistent with data, but as simple as possible in function space, embodying Occam's razor and avoiding overfitting. The effective dimensionality explains the number of parameters that have been determined by the data, which corresponds to the number of parameters the model is using to make predictions. Therefore in comparing models of the same parameterization that achieve low loss on the training data, we expect models with \emph{lower} effective dimensionality to generalize better --- which is empirically verified in Figures  \ref{fig:nn-double-descent-intro} and \ref{fig:width_depth_exp}.

\begin{figure}[!t]
    \centering
    \includegraphics[width=\linewidth]{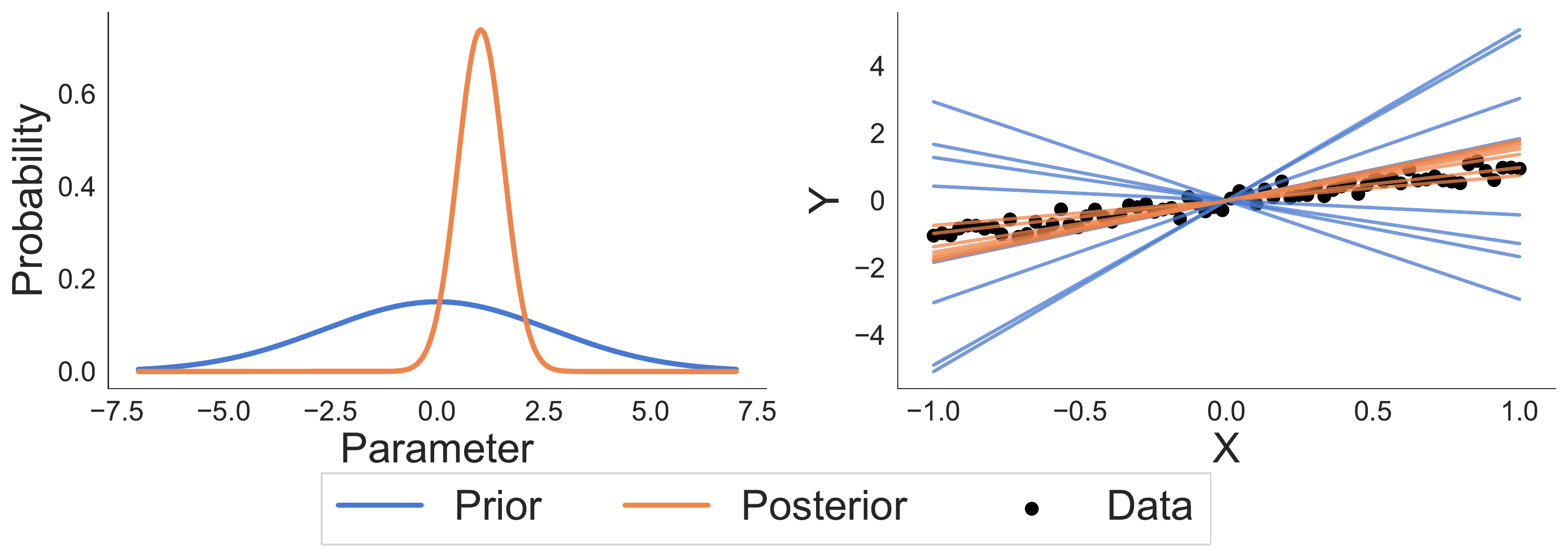}
    \caption{\textbf{Left:} A comparison of prior and posterior distributions in a Bayesian linear regression setting, demonstrating the decrease in variance referred to as \textit{posterior contraction}. \textbf{Right:} Functions sampled from the prior and posterior distributions, along with the training data.}
    \label{fig: posterior-contraction}
\end{figure}

We can further connect the Hessian and the posterior distribution by considering a Laplace approximation as in \citet{mackay1992bayesian,mackay1992interpolation}.
Here we assume that the distribution of parameters $\theta$ is multivariate normal around the maximum a posteriori (MAP) estimate, $\theta_{\text{MAP}} = \text{argmax}_{\theta}p(\theta | \mathcal{D})$,
and the Hessian of the negative log posterior, $\mathcal{H_{\theta}} + A$,\footnote{$A = -\nabla\nabla_\theta \log p(\theta)$ is the Hessian of the log prior.} serves as the precision matrix. The approximating distribution is then $\mathcal{N}(\theta_{\text{MAP}}, (\mathcal{H}_\theta+A)^{-1})$. The intuition built using Figure \ref{fig: posterior-contraction} carries through to this approximation: as the eigenvalues of the Hessian increase, the eigenvalues of the covariance matrix in our approximation to the posterior distribution shrink, further indicating contraction around the MAP estimate. 
We demonstrate this property algebraically in Appendix \ref{app:effective_dimensionality_extensions}, where we also connect the effective dimensionality to the bias-variance tradeoff  \citep{dobriban2018} and to the Hilbert space norm \citep{rasmussen_gaussian_2008}.

\subsection{Practical Computations And Parameterizations}

For deep and wide neural networks the Hessian of the loss is large, and thus computing the eigenvalues and eigenvectors is nontrivial. We employ an efficient implementation of the Lanczos method for determining eigenvalues provided in GPyTorch, allowing for the rapid computation of approximate eigenvalue, eigenvector pairs \citep{gardner2018gpytorch}. In practice, we estimate effective dimensionality by selectively computing the leading eigenvalues, since many of the eigenvalues are 
typically close to zero and do not significantly contribute to the estimate.

In general, the Hessian of the loss (and its effective dimension) is not invariant to re-parameterizations (e.g. ReLU rescaling and batch normalization) \citep[Chapter 27]{mackay2003information}. For this reason we assume a fixed parameterization, as is typically the case in practice, and compare only between models of the same 
parameterization.

\section{Related Work}

\citet{cleveland1979robust} introduced effective dimensionality into the splines literature as a measure of goodness of fit, while \citet[][Chapter 3]{hastie1990generalized} used it to assess generalized additive models.
\citet{gull1989developments} first applied effective dimensionality in a Bayesian setting for an image reconstruction task, while \citet{mackay1992bayesian,mackay1992interpolation} used it to compute posterior contraction in Bayesian neural networks. 	
\citet{moody1992effective} argued for the usage of the effective dimensionality as a proxy for generalization error, while \citet{moody1991note} suggested that effective dimensionality could be used for neural network architecture selection.
\citet{zhang2005learning} and \citet{caponnetto_optimal_2007} studied the generalization abilities of kernel methods in terms of the effective dimensionality.
		
\citet[][Chapter 7]{friedman2001elements} use the effective dimensionality (calling it the effective degrees of freedom) to compute the expected generalization gap for regularized linear models. \citet{dobriban2018} specifically tied the bias variance decomposition of predictive risk in ridge regression (e.g.\ the finite sample predictive risk under Gaussian priors) to the effective dimensionality of the feature matrix, $\Phi^\top \Phi.$
\citet{hastie_surprises_2019}, \citet{muthukumar2019harmless}, \citet{bartlett2019benign}, \citet{mei2019generalization}, and \citet{belkin2019two} studied risk and generalization in over-parameterized linear models, including under model misspecification. \citet{bartlett2019benign} also introduced the concept of effective rank of the feature matrix, which has a similar interpretation to effective dimensionality.
		
\citet{sagun2016eigenvalues} found that the eigenvalues of the Hessian increase through training, while \citet{papyan2018full} and \citet{ghorbani2019investigation} studied the eigenvalues of the Hessian for a range of modern neural networks. \citet{suzuki2018generalization} produced generalization bounds on neural networks via the effective dimensionality of the covariance of the functions at each hidden layer. 
\citet{fukumizu2019semi} embedded narrow neural networks into wider neural networks and studied the flatness of the resulting minima in terms of their Hessian via a PAC-Bayesian approach.
\citet{achille2018emergence} argue that flat minima have low information content (many small magnitude eigenvalues of the Hessian) by connecting PAC-Bayesian approaches to information theoretic arguments, before demonstrating that low information functions learn invariant representations of the data.
\citet{dziugaite2017computing} optimize a PAC-Bayesian bound to both encourage flatness and to compute non-vacuous generalization bounds, while \citet{jiang2019fantastic} recently found that PAC-Bayesian measures of flatness, in the sense of insensitivity to random perturbations, perform well relative to other generalization bounds. 
\citet{zhou2018non} used PAC-Bayesian compression arguments to construct non-vacuous generalization bounds at the ImageNet scale.

Moreover, \citet{mackay2003information} and \citet{smith2017bayesian} provide an Occam factor perspective linking flatness and generalization. Related minimum description length perspectives can be found in \citet{mackay2003information} and \citet{hinton1993keeping}.
Other works also link flatness and generalization \citep[e.g.,][]{hochreiter_flat_1997, keskar2016large, chaudhari2019entropy, izmailov2018averaging}, with \citet{izmailov2018averaging}
and \citet{chaudhari2019entropy} developing optimization procedures to select for flat regions of the loss.

To the best of our knowledge,  \citet{opper1989basins}, \citet{opper1990ability}, \citet{bos1993generalization}, and \citet{le1991eigenvalues} introduced the idea that generalization error for neural networks can decrease, increase, and then again decrease with increases in parameters (e.g.\ the double descent curve) while \citet{belkin2019reconciling} re-introduced the idea into the modern machine learning community by demonstrating its existence on a wide variety of machine learning problems.
\citet{nakkiran2019deep} found generalization gains as neural networks become highly overparameterized, showing the \emph{double descent} phenomenon that occurs as the width parameter of both residual and convolutional neural networks is increased.

\section{Posterior Contraction and Function-Space Homogeneity in Bayesian Models}\label{sec: posterior-contraction}
In this section, we demonstrate that effective dimensionality of both the posterior parameter covariance and the Hessian of the loss provides insights into how a model adapts to data during training. We derive an analytic relationship between effective dimensionality and posterior contraction for models where inference is exact, and demonstrate this relationship experimentally for deep neural networks.

\subsection{Posterior Contraction of Bayesian Linear Models}

\begin{theorem}[Posterior Contraction in Bayesian Linear Models]\label{thm: post-contraction}
Let $\Phi = \Phi(x) \in \mathbb{R}^{n \times k}$ be a feature map of $n$ data observations, $x$, with $n < k$ and assign isotropic prior $\beta \sim \mathcal{N}(0_k, \alpha^2I_k)$ for parameters $\beta \in \mathbb{R}^k$. Assuming a model of the form $y \sim \mathcal{N}(\Phi \beta, \sigma^2 I_{n})$ the posterior distribution of $\beta$ has a $k-n$ directional subspace in which the variance is identical to the prior variance.
\end{theorem}

We prove Theorem \ref{thm: post-contraction} in Appendix \ref{app: post-contraction}, in addition to an equivalent result for generalized linear models.
Theorem \ref{thm: post-contraction} demonstrates why \textit{parameter counting} often makes little sense: for a fixed data set of size $n$, only $\min(n, k)$ parameters can be determined, leaving many dimensions in which the posterior is unchanged from the prior when $k \gg n$. 

\paragraph{Empirical Demonstration for Theorem \ref{thm: post-contraction}.} 
We construct $\Phi(x)$ with each row as an instance of a $200$ dimensional feature vector consisting of sinusoidal terms for each of $500$ observations: $\Phi(x) = [\cos(\pi x), \sin(\pi x), \cos(2\pi x), \sin(2\pi x),\dots]$.
We assign the coefficient vector $\beta$ a prior $\beta \sim \mathcal{N}(0, I)$, and draw ground truth parameters $\beta^*$ from this distribution. The model takes the form $\beta \sim \mathcal{N}(0, I)$ and $y \sim \mathcal{N}\left(\Phi \beta, \sigma^2 I\right)$.

We randomly add data points one at a time, tracking the posterior covariance matrix at each step. We compute the effective dimensionality, $N_{eff}\left(\Sigma_{\beta | \mathcal{D},  \sigma}, \alpha\right)$, where $\Sigma_{\beta | \mathcal{D},  \sigma}$ is the posterior covariance of $\beta$.\footnote{Here we use $\alpha = 5$, however the results remain qualitatively the same as this parameter changes.} 

In Figure \ref{fig: bayes-eff-dim} we see that the effective dimensionality of the posterior covariance decreases linearly with an increase in available data until the model becomes overparameterized, at which point the effective dimensionality of the posterior covariance of the parameters slowly approaches $0$, while the effective dimensionality of the Hessian of the loss increases towards an asymptotic limit.
As the parameters become more determined (e.g. the effective dimensionality of the posterior covariance decreases), the curvature of the loss increases (the effective of the Hessian increases).
In the Bayesian linear model setting, the Hessian of the loss is the inverse covariance matrix and the trade-off between  the effective dimensionality of the Hessian and the parameter covariance can be determined algebraically (see Appendix \ref{app:neff_proof}). 

\begin{figure}[!t]
    \centering
    \includegraphics[width=\linewidth]{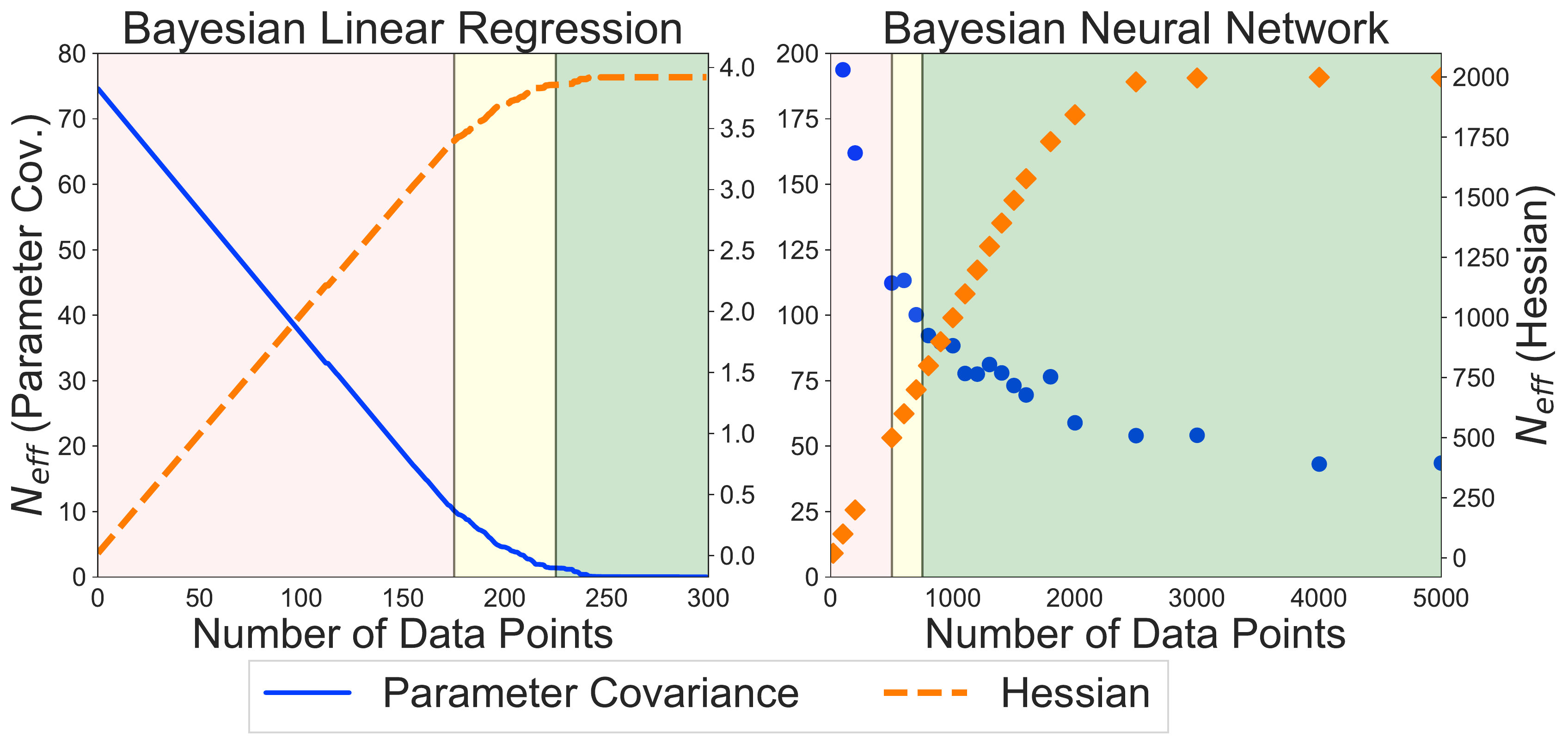}
    \caption{\textbf{Left:} Bayesian linear regression. \textbf{Right:} Bayesian neural network. \textbf{Both:} The effective dimensionality of the posterior covariance over parameters and the function-space posterior covariance. Red indicates the under-parameterized setting, yellow the critical regime with $p \approx n$, and green the over-parameterized regime. In both models we see the expected increase in effective dimensionality in parameter space and decrease in effective dimensionality of the Hessian.}
    \label{fig: bayes-eff-dim}
\end{figure}

\subsection{Posterior Contraction of Bayesian Neural Networks}
\label{sec: post-con-bnn}

While much effort has been spent grappling with the challenges of marginalizing a high dimensional parameter space for Bayesian neural networks, the practical existence of subspaces where the posterior variance has not collapsed from the prior suggests that both computational and approximation gains can be made from ignoring directions in which the posterior variance is unchanged from the prior. 
This observation helps explain the success of subspace based techniques that examine the loss in a lower dimensional space such as \citet{izmailov2019subspace}. Alternatively, by working directly in function space, as in \citet{sun2019functional}, the redundancy of many parameters could be avoided.

For Bayesian linear models, the effective dimensionality of the parameter covariance is the inverse of the Hessian, and as the effective dimensionality of the parameter covariance decreases the effective dimensionality of the Hessian increases.
We hypothesize that a similar statement holds for Bayesian neural networks --- as the number of data points grows, the effective dimensionality of the posterior covariance should decrease while the effective dimensionality of the Hessian should increase.

To test this hypothesis, we generate a nonlinear function of the form, $y = w_1 x + w_2 x^2 + w_3 x^3 + (0.5 + x^2)^2 + \text{sin}(4 x^2) + \epsilon,$ with $w_i \sim \mathcal{N}(0,I)$ and $\epsilon \sim \mathcal{N}(0, 0.05^2),$ and de-mean and standardize the inputs.\footnote{From the Bayesian neural network example in NumPyro \citep{phan2019composable,bingham2019pyro}: \url{https://github.com/pyro-ppl/numpyro/blob/master/examples/bnn.py}.}
We then construct a Bayesian neural network with two hidden layers each with $20$ units, no biases, and $tanh$ activations, placing independent Gaussian priors with variance $1$ on all model parameters. We then run the No-U-Turn sampler \citep{hoffman2014no} for $2000$ burn-in iterations before saving the final $2000$ samples from the approximated posterior distribution.
Using these samples, we compute the effective dimensionality of the sample posterior covariance, $\text{Cov}_{p(\theta | \mathcal{D})}(\theta)$, and Hessian of the loss at the MAP estimate in Figure \ref{fig: bayes-eff-dim}.
The trends of effective dimensionality for Bayesian neural networks are aligned with Bayesian linear regression, with the effective dimensionality of the Hessian (corresponding to function space) increasing while the effective dimensionality of the parameter space decreases.

\subsection{Function-Space Homogeneity}
\label{sec: func-homog-blr}

In order to understand how the function-space representation varies as parameters are changed in directions \emph{undetermined} by the data, we first consider Bayesian linear models.

\begin{theorem}[Function-Space Homogeneity in Linear Models]\label{thm: function-homog}
Let $\Phi = \Phi(x) \in \mathbb{R}^{n \times k}$ be a feature map of $n$ data observations, $x$, with $n < k$, and assign isotropic prior $\beta \sim \mathcal{N}(0_k, \alpha^2I_k)$ for parameters $\beta \in \mathbb{R}^k$.
The minimal eigenvectors of the Hessian define a $k-n$ dimensional subspace in which parameters can be perturbed without changing the training predictions in function space.
\end{theorem}

We prove Theorem \ref{thm: function-homog} and its extension to generalized linear models in Appendix \ref{app: function-homog}. This theorem suggests that although there may be large regions in parameter-space that lead to low-loss models, many of these models may be homogeneous in function space.

\begin{figure}[!t]
    \centering
    \includegraphics[width=\linewidth]{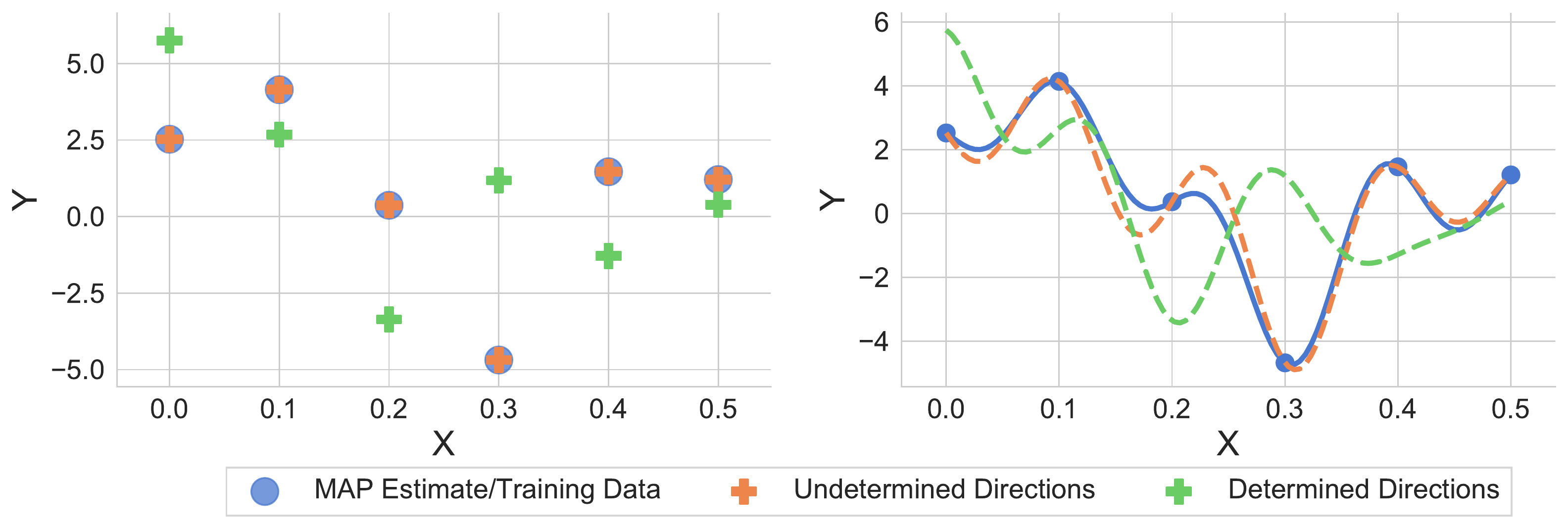}
    \caption{\textbf{Left:} The predictions on training data for a simple Bayesian linear regression model with sinusoidal features for various parameter settings. \textbf{Right:} The predictions over the entire test domain. \textbf{Both:} Blue represents the MAP estimate as well as the training points, orange represents the model after the parameters have been perturbed in a direction in which the posterior has not contracted, and green represents the model after parameters have been perturbed in a direction in which the posterior has contracted. Perturbing parameters in directions that have not been determined by the data gives not only identical predictions on training data, but the functions produced on the test set are nearly the same.}
    \label{fig: blr-func-homog}
    \vspace{-0.3cm}
\end{figure}

We can interpret Theorem \ref{thm: function-homog} in terms of the eigenvectors of the Hessian indicating which directions in parameter space have and have not been determined by the data. The dominant eigenvectors of the Hessian (those with the largest eigenvalues) correspond to the directions in which the parameters have been determined from the data and the posterior has contracted significantly from the prior. The minimal eigenvectors (those with the smallest eigenvalues) correspond to the directions in parameter space in which the data has not determined the parameters. 

Figure \ref{fig: blr-func-homog} demonstrates the result of Theorem \ref{thm: function-homog} for a Bayesian linear model with sinusoidal features. We compare predictions made using the MAP estimate of the parameters, $\theta^* = \text{argmax}_{\theta}p(\theta | \mathcal{D})$, to predictions generated using perturbed parameters.
As parameters are perturbed in directions that have not been determined by the data (minimal eigenvectors of the Hessian), the predictions on
both train and test remain nearly identical to those generated using the MAP estimate.
Perturbations in determined directions (dominant eigenvectors of the Hessian) yield models that perform poorly on the training data and significantly deviate from the MAP estimate on the test set.

\section{Loss Surfaces and Function Space Representations}
\label{sec: loss-surfaces}

Recent works have discussed the desirability of finding solutions corresponding to \emph{flat} optima in the loss surface, arguing that such parameter settings lead to better generalization \citep{izmailov2018averaging, keskar2016large}.
There are multiple notions of flatness in loss surfaces, relating to both the volume of the basin in which the solution resides and the rate of increase in loss as one moves away from the found solution. As both definitions correspond to low curvature in the loss surface, it is standard to use the Hessian of the loss to examine structure in the loss surface \cite{madras2019detecting, keskar2016large}.

The effective dimensionality of the Hessian of the loss indicates the number of parameters that have been determined by the data. In highly over-parameterized models we hypothesize that the effective dimensionality is substantially less than the number of parameters, i.e. $N_{eff}(\mathcal{H}_{\theta}, \alpha) \ll p,$ since we should be unable to determine many more parameters than we have data observations.

Recall from Section \ref{sec: eff-dim} the large eigenvalues of the Hessian have eigenvectors corresponding to directions in which parameters are determined. Eq.~\eqref{eqn: eff-dim} dictates that low effective dimensionality (in comparison to the total number of parameters) would imply that there are many directions in which parameters are not determined, and the Hessian has eigenvalues that are near zero, meaning that in many directions the loss surface is constant. We refer to directions in parameter space that have not been determined as \emph{degenerate} for two reasons: (1) degenerate directions in parameter space provide minimal structure in the loss surface, shown in Section \ref{sec: hessian-loss-surface}; (2) parameter perturbations in degenerate directions do not provide diversity in the function-space representation of the model, shown in Section \ref{sec: func-homog-nn}.
We refer to the directions in which parameters have been determined, directions of high curvature, as \textit{determined}.

To empirically test our hypotheses regarding degenerate directions in loss surfaces and function space diversity, we train a neural network classifier on $1000$ points generated from the two-dimensional Swiss roll data, with a similar setup to \citet{huang2019understanding}, using Adam with a learning rate of $0.01$ \cite{kingma2014adam}. The network is fully connected, consisting of $5$ hidden layers each $20$ units wide (plus a bias term), and uses ELU activations with a total of $2181$ parameters. We choose a small model with two-dimensional inputs so that we can both tractably compute all the eigenvectors and eigenvalues of the Hessian and visualize the functional form of the model. To demonstrate the breadth of these results, we provide comparable visualizations in the Appendix \ref{app: cifar}, but for a convolutional network trained on CIFAR-$10$.

\subsection{Loss Surfaces as Determined by the Hessian}
\label{sec: hessian-loss-surface}
To examine the loss surface more closely, we visualize low dimensional projections. To create the visualizations, we first define a basis given by a set of vectors, then choose a two random vectors, $u$ and $\widetilde{v}$, within the span of the basis. We use Gram-Schmidt to orthogonalize $\widetilde{v}$ with respect to $u$, ultimately giving $u$ and $v$ with $u \perp v$. We then compute the loss at parameter settings $\theta$ on a grid surrounding the optimal parameter set, $\theta^*$, which are given by
\begin{equation}
    \theta \leftarrow \theta^* + \alpha u + \beta v
\end{equation}
for various $\alpha$ and $\beta$ values such that all points on the grid are evaluated.

By selecting the basis in which $u$ and $v$ are defined we can specifically examine the loss in determined and degenerate directions. Figure \ref{fig: loss-surfaces} shows that in determined directions, the optimum appears extremely sharp. Conversely, in all but the most determined directions, the loss surface loses all structure and appears constant.
Even in degenerate directions, if we deviate from the optimum far enough the loss will eventually become large. However to observe this increase in loss requires perturbations to the parameters that are significantly larger in norm than $\theta^*$.

\begin{figure}[!t]
    \centering
    \includegraphics[width=\linewidth]{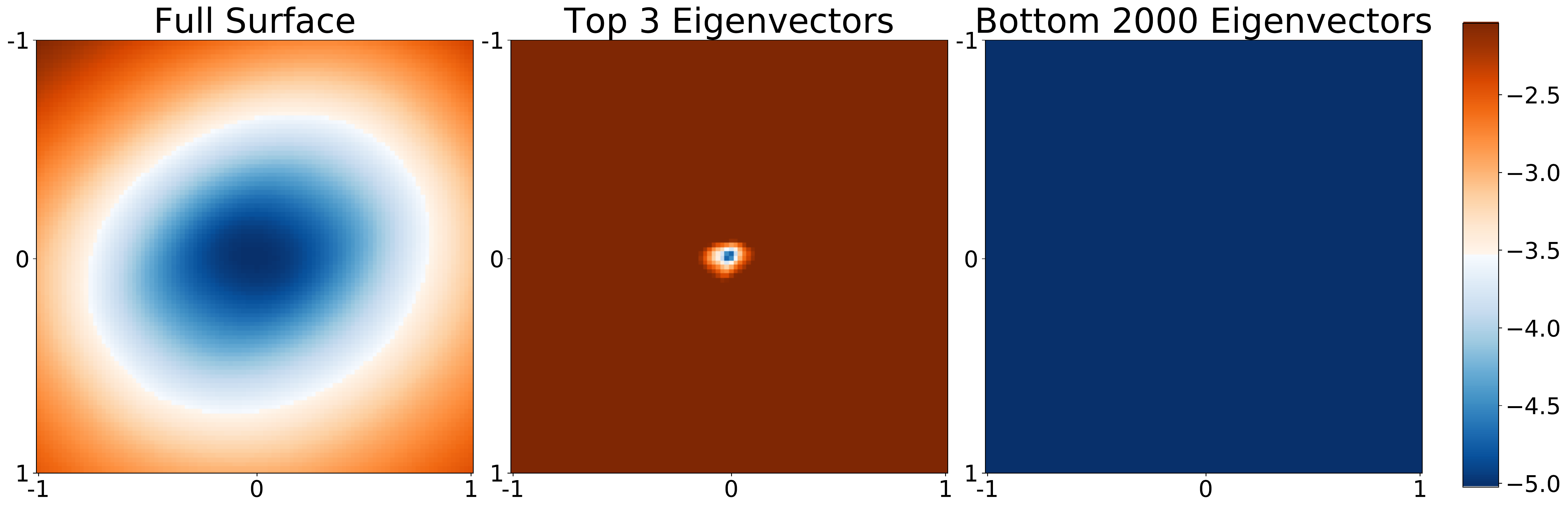}
    \caption{\textbf{Left:} A random projection of the loss surface. \textbf{Center:} A projection of the loss surface in the top $3$ directions in which parameters have been determined. \textbf{Right:} A projection of the loss surface in the $2000$ (out of $2181$) directions in which parameters have been determined the least. The rightmost plot shows that in degenerate parameter directions the loss is constant.}
    \label{fig: loss-surfaces}
\end{figure}

\subsection{Degenerate Parameters Lead to Homogeneous Models}
\label{sec: func-homog-nn}
In this section we show that degenerate parameter directions do not contain diverse models. This result is not at odds with the notion that flat regions in the loss surface can lead to diverse but high performing models. Rather, we find that there is a subspace in which the loss is constant and one cannot find model diversity, noting that this subspace is distinct from those employed by works such as \citet{izmailov2019subspace} and \citet{huang2019understanding}. This finding leads to an interpretation of effective dimensionality as \emph{model compression}, since the undetermined directions do not contain additional functional information.

We wish to examine the functional form of models obtained by perturbing the parameters found through training, $\theta^*$. Perturbed parameters are computed as
\begin{equation}
    \mathbf{\theta} \leftarrow \mathbf{\theta}^* +  s \frac{B v}{||B v||_{2}}
\end{equation}
where $B \in \mathbb{R}^{k \times d}$ is a $d$ dimensional basis in which we wish to perturb $\theta^*$, and $v \sim \mathcal{N}(0, I_d)$, giving $Bv$ as a random vector from within the span of some specified basis (i.e. the dominant or minimal eigenvectors). The value $s$ is chosen to determine the scale of the perturbation, i.e. the length of the random vector by which the parameters are perturbed.

Experimentally, we find that in a region near the optimal parameters $\theta^*$, i.e. $s \leq ||\theta^*||_{2}/2$  the function-space diversity of the model is contained within the subspace of determined directions. While the degenerate directions contain wide ranges of parameter settings with low loss, the models are equivalent in function space.

Figure \ref{fig: perturbed-classifier} shows the trained classifier and the differences in function-space between the trained classifier and those generated from parameter perturbations. We compare perturbations of size $||\theta^*||_{2}/2 \approx 10$ in the direction of the $500$ minimal eigenvectors and perturbations of size $0.1$ in the directions of the $3$ maximum eigenvectors. A perturbation from the trained parameters in the directions of low curvature (center plot in Figure \ref{fig: perturbed-classifier}) still leads to a classifier that labels all points identically. A perturbation roughly $100$ times smaller the size in directions in which parameters have been determined leads to a substantial change in the decision boundary of the classifier.

However, the change in the decision boundary resulting from perturbations in determined directions is not necessarily desirable. One need not perturb parameters in either determined or degenerate directions to perform a downstream task such as ensembling. Here, we are showcasing the  degeneracy of the subspace of parameter directions that have not been determined by the data. This result highlights that despite having many parameters the network could be described by a relatively low dimensional subspace.

\begin{figure}[!t]
    \centering
    \includegraphics[width=\linewidth]{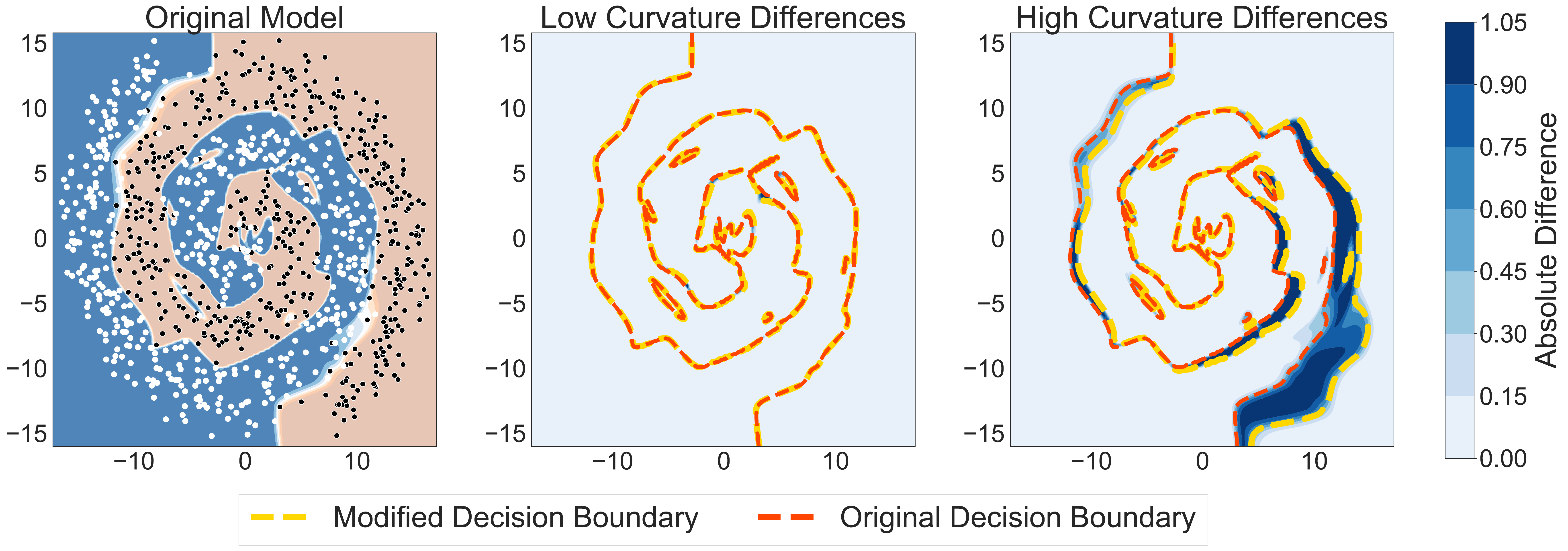}
    \caption{Swiss roll data. \textbf{Left:} Adam trained feed-forward, fully connected classifier. \textbf{Center:} Differences in original and perturbed classifier when parameters are perturbed by in low curvature, degenerate directions.
    \textbf{Right:} Differences in the original and perturbed classifier when parameters are perturbed in high curvature directions. \textbf{Note} the perturbation in the center plot is approximately $100$ times the size of that of the plot on the right.}
    \label{fig: perturbed-classifier}
\end{figure}

\subsection{Effective Dimensionality as Compression}
\label{sec: compression}

In Section \ref{sec: func-homog-nn} we showed that effective dimensionality relates to the number of parameter directions in which the functional form of the model is sensitive to perturbations, and that in the low curvature directions given by the Hessian eigenvectors with smallest eigenvalues the model outputs are largely unchanged by perturbations to the parameters. 
The presence of these degenerate directions suggests that we can disregard high dimensional subspaces that contain little information about the model, for compression into a smaller subspace containing only the most important parameter directions given by the eigenvalues of the Hessian. This observation helps explain the practical success of such subspace approaches 
\citep{izmailov2019subspace, li2018measuring}.

We can also understand the compression of the data provided by a model in terms of minimum description length, by examining the \emph{Occam factor} \citep[Chapter~28]{mackay2003information}.
For model $\mathcal{M}$ with parameters $\theta$, we find the Occam factor in decomposing the evidence as,
\begin{equation}\label{eqn: occam}
\begin{aligned}
    p(\mathcal{D} | \mathcal{M}_i) &\approx \underbrace{p(\mathcal{D} | \theta_{MP}, \mathcal{M})} \times 
                                 \underbrace{p(\theta_{MP} | \mathcal{M})\textrm{det}^{-\frac{1}{2}}(\mathcal{H}_\theta/2\pi)}, \\
    \textrm{Evidence} &\approx \quad \textrm{Likelihood}\;\;\times \qquad \textrm{Occam Factor}
\end{aligned}
\end{equation}
in which $\mathcal{H}_{\theta}$ is the Hessian of the loss, and $\theta_{MP}$ is the \emph{maximum a posteriori} estimate of the parameters.

As the eigenvalues of the Hessian decay and the effective dimensionality decreases, the determinant of the Hessian also decreases, causing the Occam factor to increase, and the description length to decrease, providing a better compression.
\citet{mackay1992interpolation} and \citet{mackay2003information} contains a further discussion of the connection between Occam factors and minimum description length.
These connections also help further explain the practical success of optimization procedures that select for flat regions of the loss \citep{izmailov2018averaging, keskar2016large}.

\section{Double Descent and Effective Dimensionality}
\label{sec: double-descent}
The principle of Occam's razor suggests we want a model that is as simple as possible while still fitting the training data. By the same token we ought to desire a model with low effective dimensionality. Low effective dimensionality indicates that the model is making full use of a smaller number of parameters, providing a better compression of the data, and thus likely better generalization.
The phenomenon of \emph{double descent} of the generalization performance in both linear and deep models has attracted recent attention \citep{nakkiran2019deep,belkin2019reconciling}; here, we explain double descent by effective dimensionality.

We find that for models in which the training loss converges to near zero, the effective dimensionality corresponds remarkably well to generalization performance, despite having been determined only from training data. 
For models that are only just able to achieve zero training error, but generalize poorly due to overfitting, the effective dimension is high. In these cases high effective dimensionality is due the sensitivity of the fit to the precise settings of the parameters.
As the model changes and grows, there exist a greater variety of subspaces which provide more effective compressions of the data, and thus we achieve a lower effective dimensionality.
We demonstrate the correspondence of effective dimensionality to generalization performance in the regime with near-zero training loss for both linear models and deep neural networks. 

\emph{In short, double descent is an artifact of overfitting. As the dimensionality of the parameter space continues to increase past the point where the corresponding models achieve zero training error, flat regions of the loss occupy a greatly increasing volume \citep{huang2019understanding}, and are thus more easily discoverable by optimization procedures such as SGD. These solutions have lower effective dimensionality, and thus provide better lossless compressions of the data, as in Section~\ref{sec: compression}, and therefore better generalization. In concurrent work, \citet{wilson2020bayesian} show that exhaustive Bayesian model averaging over multiple modes eliminates double descent.}

\subsection{Double Descent on Linear Models}
Although double descent is often associated with neural networks, we here demonstrate similar behaviour with a linear model with a varying number of features: first drawing $200$ data points $y \sim \mathcal{N}(0, 1)$ and then drawing $20$ informative features $y + \epsilon,$ where $\epsilon \sim \mathcal{N}(0, 1),$ before drawing $k - 20$ features that are also just random Gaussian noise, where $k$ is the total number of features in the model.\footnote{From \url{https://github.com/ORIE4741/demos/blob/master/double-descent.ipynb}.} For the test set, we repeat the generative process. In Figure \ref{fig:weak-double-descent} we show a pronounced double descent curve in the test error as we increase the number of features, which is mirrored by the effective dimensionality.

\begin{figure}[!t]
    \centering
    \includegraphics[width=\linewidth]{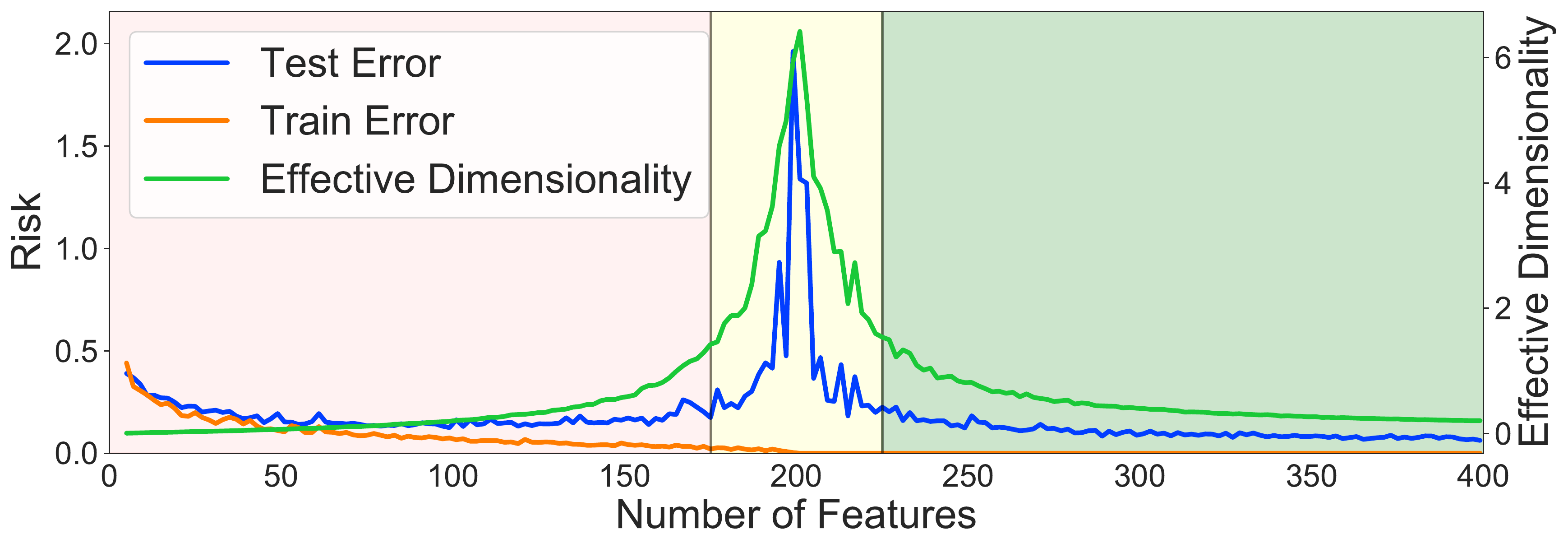}
    \caption{Demonstration of double descent for linear models with an increasing number of features. We plot the effective dimensionality of the Hessian of the loss. In the regime with near-zero train error, the test error is almost entirely explained by the effective dimensionality, which is computed on the train set alone. The red region corresponds to underparameterized models, yellow to critically parameterized models, and green to overparameterized models.}
    \label{fig:weak-double-descent}
\end{figure}

\subsection{Double Descent for Deep Models}

We demonstrate and explain double descent as a function of the effective dimensionality for deep neural networks.
\citet{nakkiran2019deep} demonstrated that double descent can also occur for modern deep neural networks including transformers, CNNs, and ResNets.
Following their experimental setup, we train ResNet18s \citep{he2016identity} with varying width parameters, reproducing the double descent curve shown by \citet{nakkiran2019deep}.\footnote{See Appendix \ref{app:dnn_exp_details} for training details.}
We compute the effective dimensionality of the model using $100$ eigenvalues as calculated from a GPU accelerated Lanczos in GPyTorch \citep{gardner2018gpytorch}. 
In Figure \ref{fig:nn-double-descent-intro} we see effective dimensionality tracks remarkably well with generalization --- displaying the double descent curve that is seen in the test error. We emphasize again that the effective dimensionality is computed using solely the training data, supporting the hypothesis that the eigenvalues of the Hessian matrix can provide a good proxy for generalization performance.
In Appendix \ref{app: classifiers}, we test small neural networks on a problem for which we can compute all of the eigenvalues replicating a similar finding to Figure \ref{fig:nn-double-descent-intro} as width increases.

In addition to test loss, we also demonstrate that effective dimensionality tracks double descent in \emph{test error} in Figure~\ref{fig:dd_test_error} (Appendix). Double descent is clearly present for test loss on CIFAR-100, but not test error. As we discussed at the beginning of Section~\ref{sec: double-descent}, double descent is an artifact of overfitting. To produce double descent for test error, we follow the setup in \citet{nakkiran2019deep} and introduce $20\%$ label corruption, which increases the chance of overfitting. We also show additional results for double descent in Figures~\ref{fig: ddnolabel} and \ref{fig: ddlabel}, where we compare with other generalization measures that are described in Section~\ref{sec: generalization}.

\subsection{Networks of Varying Width and Depth}

Double descent experiments typically only consider increases in width. However, it is \emph{depth} which has endowed neural networks 
with distinctive generalization properties.
In Figure \ref{fig:width_depth_exp}, we consider varying both the width and depth of a 
convolutional neural network on the CIFAR-$100$ dataset. We measure effective dimensionality, training loss, and testing loss. The yellow curves show networks with a constant number of parameters, indicating the simple parameter counting is not a good proxy for generalization. However, in the region of near-zero training loss, separated by the green curve, we see effective dimensionality closely matches generalization performance. Moreover, wide but shallow models tend to overfit, providing low training loss, but high effective dimensionality and test loss. On the other hand, deeper models have lower test loss and lower effective dimensionality, showing that depth enables a better compression of the data.

\section{Effective Dimensionality as a Generalization Measure}
\label{sec: generalization}

We have shown, for the first time, that a generalization measure is able to track and explain double descent and width-depth trade-offs in modern deep networks. Here we compare effective dimensionality with alternative norm- and flatness-based measures. We choose path-norm and a PAC-Bayesian based sharpness measure, due to their good performance on other problems in prior work \citep{jiang2019fantastic, keskar2016large, neyshabur2017exploring}. We compute the norms as in \citet{jiang2019fantastic} for a neural network function $f(x; \theta)$ with input $x$ and weights $\theta$. 

The path-norm is the square root of the sum of the outputs produced by a forwards pass on an input of all ones,
\begin{equation}
    \mu_{\textrm{path-norm}}(f) = \left(\sum f(\mathbf{1}; \theta^2) \right)^{1/2} \,.
\end{equation}
with the parameters $\theta$ squared (Eq. 44 of \citet{jiang2019fantastic}).
Through squaring the weights and taking the square root of the output, we form a correspondence between the path-norm and the $\ell_2$ norm of all paths within a network from an input node to an output node \citep{neyshabur2017exploring}.

The PAC-Bayesian flatness measure of \citet{jiang2019fantastic}, adapted from \citet{dziugaite2017computing} and \citet{keskar2016large}, is perturbation-based and computed as 
\begin{equation}\label{eqn: pacbayes}
   \mu_{\textrm{pac-bayes-sharpness}}(f) = \frac{1}{\sigma^2}  \,,
\end{equation}
where $\sigma$ is the largest value such that
\begin{equation}\label{eqn: pacbayes-expec}
    \mathbb{E}_{u \sim \mathcal{N}(0, \sigma^2 I)} \left[ \mathcal{L}(\theta + u, \mathcal{D})\right] \leq \mathcal{L}(\theta, \mathcal{D}) + 0.1 \,.
\end{equation}
In Equation \ref{eqn: pacbayes-expec}, $\mathcal{L}(\theta, \mathcal{D})$ is the prediction \emph{error} on the training dataset of the network with weights $\theta$ as computed on data set $\mathcal{D}$. This measure corresponds to a bound on parameter perturbations such that increases in training error remain beneath $0.1$ in expectation as in \citet{jiang2019fantastic}.

We additionally compare to the magnitude aware PAC-Bayes bound,
\begin{equation}\label{eqn: mag-pacbayes}
   \mu_{\textrm{mag-pac-bayes-sharpness}}(f) = \frac{1}{\sigma'^2} \,,
\end{equation}
where $\sigma'$ is the largest value such that
\begin{equation}\label{eqn: mag-pacbayes-expec}
    \mathbb{E}_{u \sim \mathcal{N}(0, \sigma'^2 |\theta| + \epsilon)}  \left[ \mathcal{L}(\theta + u, \mathcal{D})\right] \leq  \mathcal{L}(\theta, \mathcal{D}) + 0.1 \,.
\end{equation}
In Equation \ref{eqn: mag-pacbayes-expec} the variance of the perturbation to each parameter is scaled according to the magnitude of that parameter, adding stability by accounting for differences in scales, and, implicitly, the size of the perturbation with the dimension of the parameter space. The value of $\epsilon$ is taken to be $0.001$ as in \citet{jiang2019fantastic} and serves to regularize the distribution, preventing the distribution from collapsing in the presence of weights that are close to $0$.

We extend the results of Figures \ref{fig:nn-double-descent-intro} and \ref{fig:width_depth_exp} in Figure \ref{fig: generalization}, for double descent and wide-depth trade-offs, to include the path-norm and PAC-Bayesian flatness measures. We display test loss, test error, and generalization measures, standardized by subtracting the sample mean and dividing by the sample standard deviation. We additionally show correlations with generalization in Tables \ref{tab: gen-correlation-dd} and \ref{tab: gen-correlation-wd}.

\begin{table}
\centering
\begin{tabular}{ c|c|c|c|} 
  & Test loss & Test Error & Gen. Gap\\ 
  \hline
  $N_{eff}$(Hessian) & $\mathbf{0.9434}$ & $0.9188$ &  $\mathbf{0.9429}$\\
  \hline
   PAC-Bayes & $-0.8443$ & $-0.7372$ & $-0.8597$\\ 
 \hline
 Mag. PAC-Bayes & $0.7066$ & $0.8270$ & $0.6805$ \\ 
 \hline
 Path-Norm & $0.5598$ & $0.7216$ & $0.5259$ \\ 
 \hline
 Log Path-Norm & $0.9397$ & $\mathbf{0.9846}$ & $0.9257$ \\ 
 \hline

\end{tabular}
\caption{Sample Pearson correlation with generalization on double descent for ResNet$18$s of varying width on CIFAR-$100$ with a training loss below $0.1$.
}
\label{tab: gen-correlation-dd}
\end{table}

\begin{table}
\centering
\begin{tabular}{ c|c|c|c|} 
  & Test loss & Test Error & Gen. Gap\\ 
  \hline
  $N_{eff}$(Hessian) & $0.9305$ & $\mathbf{0.9461}$ &  $0.9060$\\
  \hline
   PAC-Bayes & $-0.8619$ & $-0.7916$ & $-0.8873$\\ 
 \hline
 Mag. PAC-Bayes & $0.8724$ & $0.9225$ & $0.8330$ \\ 
 \hline
 Path-Norm & $0.7996$ & $0.7721$ & $0.7511$ \\ 
 \hline
 Log Path-Norm & $\mathbf{0.9781}$ & $0.9402$ & $\mathbf{0.9602}$ \\ 
 \hline

\end{tabular}
\caption{Sample Pearson correlation with generalization for CNNs of varying width and depth on CIFAR-$100$ with a training loss below $0.1$.}
\label{tab: gen-correlation-wd}
\end{table}

\begin{figure*}%
    \centering
    \begin{subfigure}{1.0\columnwidth}
            \includegraphics[width=\columnwidth]{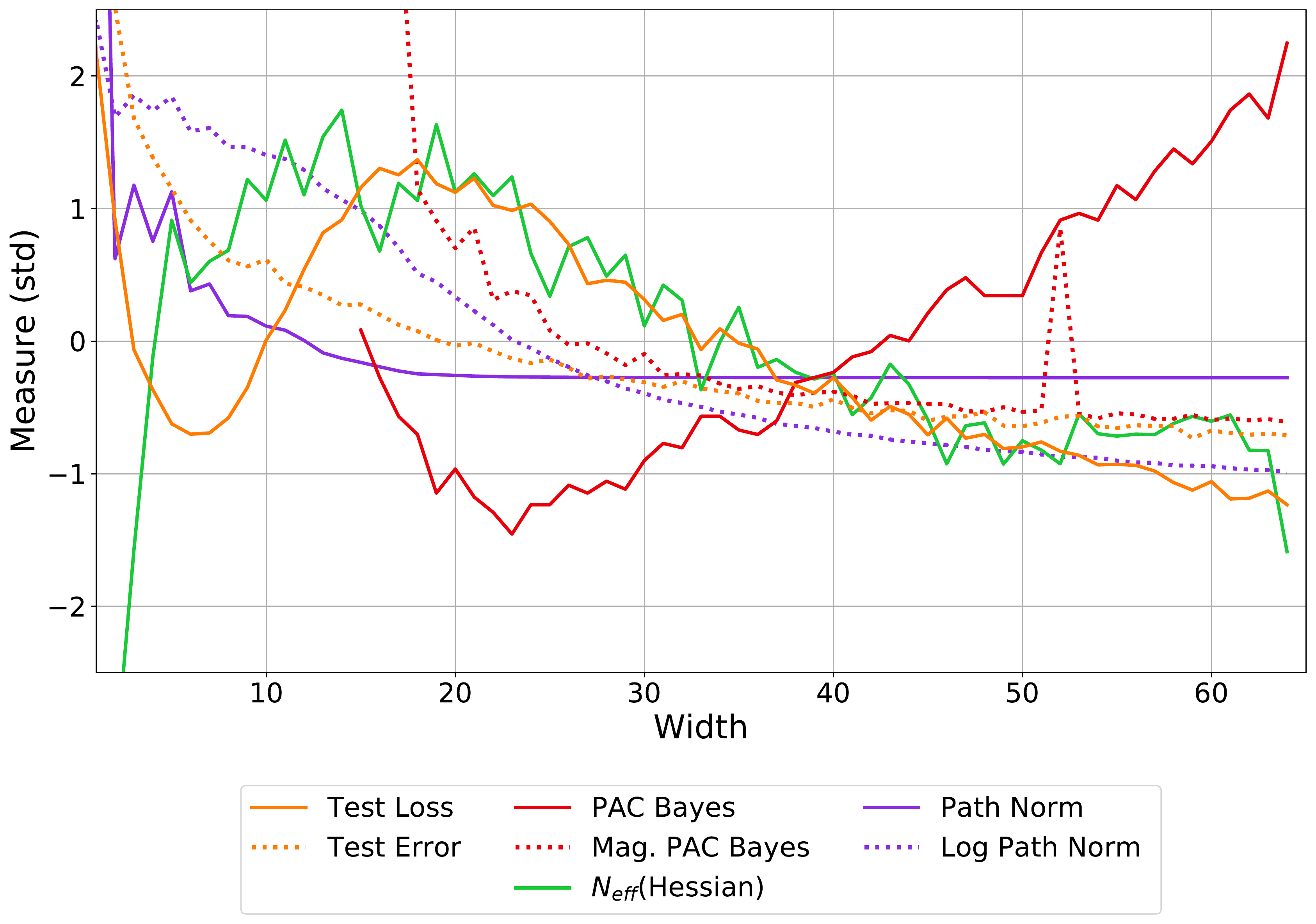}%
        \caption{\textit{Double Descent, No Label Noise}}
        \label{fig: ddnolabel}
    
    \end{subfigure}\hfill%
    \begin{subfigure}{1.0\columnwidth}
        \includegraphics[width=\columnwidth]{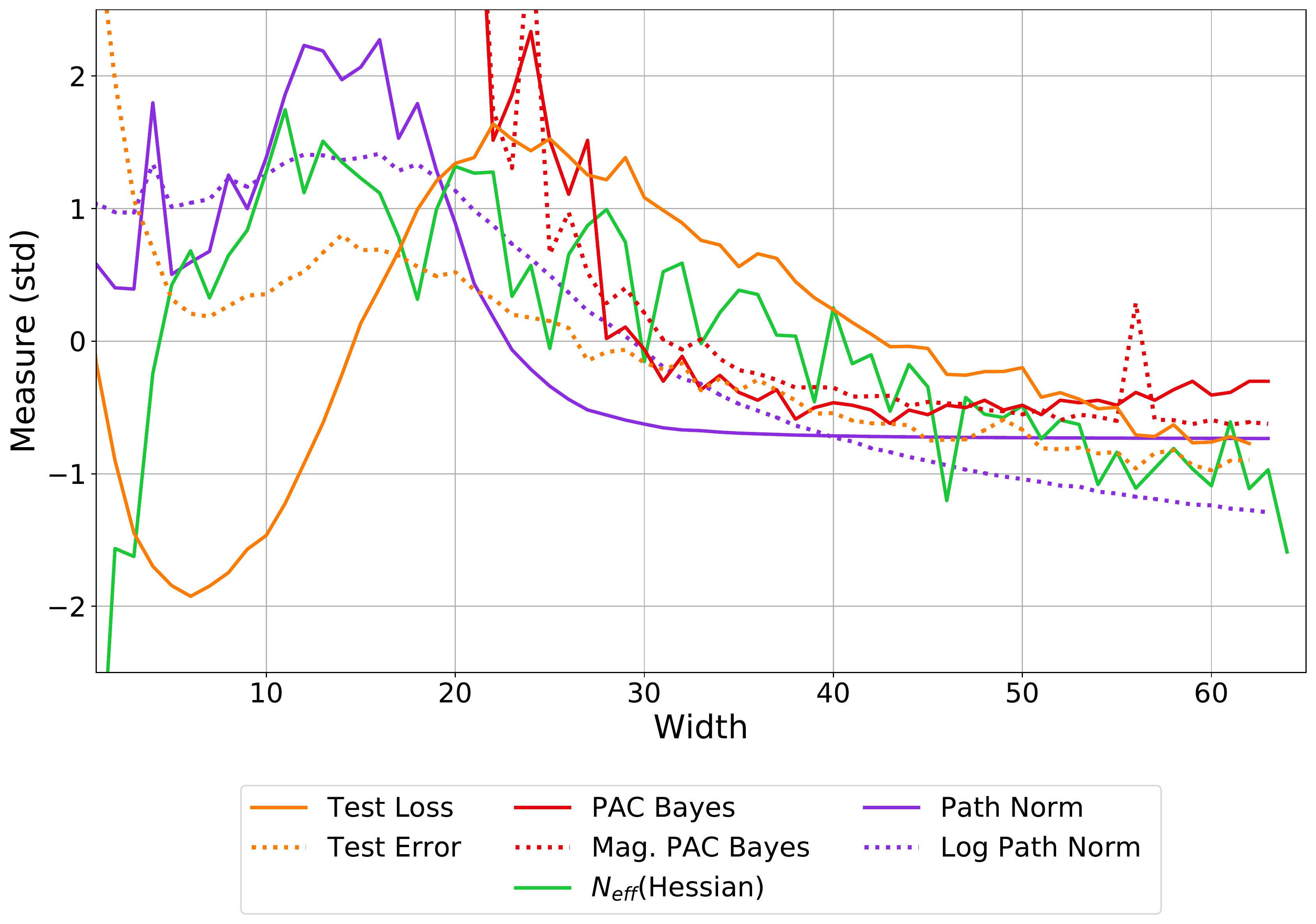}
        \caption{\textit{Double Descent, 20\% Label Noise}}
        \label{fig: ddlabel}
    \end{subfigure}\\[1ex]
    \begin{subfigure}{\columnwidth}
        \includegraphics[width=\columnwidth]{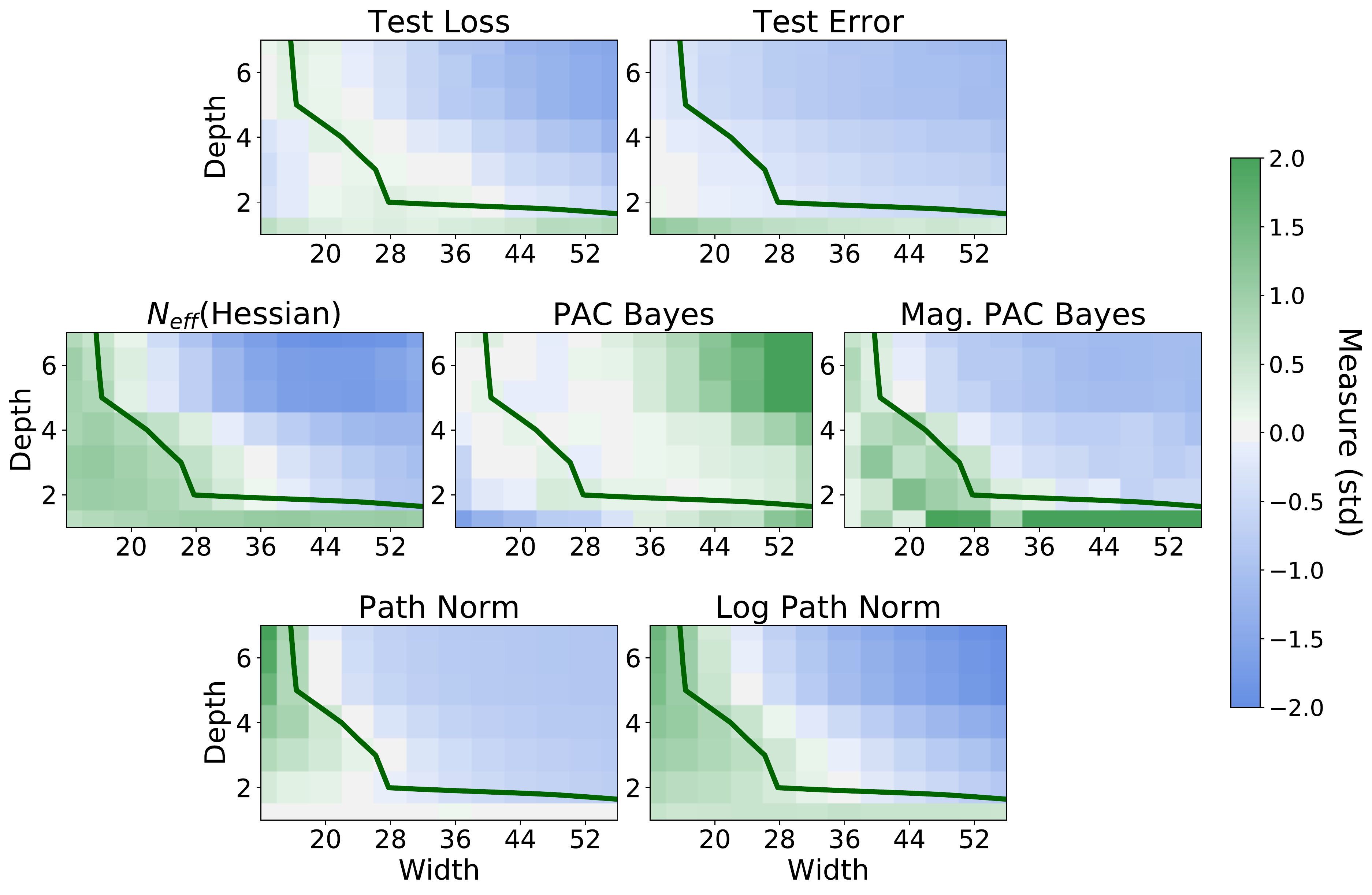} 
        \caption{\textit{Width vs. Depth}}
        \label{fig: width-depth-generalization}
    \end{subfigure}

    \caption{Comparing effective dimensionality as a generalization measure to the path-norm, log path-norm, PAC-Bayes, and magnitude aware PAC-Bayes flatness measure from \citet{jiang2019fantastic} for double descent without (a) and with (b) label noise, and (c) width-depth trade-offs. We find that among models that achieve low train loss, effective dimensionality most closely follows the trends in both test error and test loss. Path-norm is very large for small models where generalization performance is poor, and the PAC-Bayes measure grows with model size. In panel (c) the green curve separates models that achieve below $0.1$ training error.
    Additionally, we standardize all measures and test error to $0$ mean and unit variance for comparison.}
    \label{fig: generalization}
\end{figure*}

There are several key take-aways: (1) effective dimension overall provides a better proxy for test loss and test error than the PAC-Bayes or path-norm measures for double descent and width-depth trade-offs; (2) not all flatness-based generalization measures are equivalent, and in fact different flatness measures can provide wildly different behaviour; (3) effective dimension is more stable than the other measures, with relatively consistent behaviour across qualitatively similar datasets and models; (4) effective dimension is \emph{interpretable}, providing a clear connection with function-space and an explanation for \emph{why} a model should generalize, in addition to an association with generalization.

In general, path-norm acts only on model parameters, and is thus not directly connected to function-space or the shape of the loss. It should therefore be used with particular caution in comparing different architectures. Indeed, we found that path-norm tends to quickly saturate with increases in model size, with no preference between larger models even when these models provide very different generalization performance. We exhaustively considered how path-norm could be made to provide reasonable performance on these problems. For smaller models, the path-norm is orders of magnitude larger than the path-norm computed on larger models. We therefore performed a log transform, which to our knowledge has not been considered before. We surprisingly found the log transform dramatically improves the correlation of path-norm to generalization performance for comparing amongst large convolutional nets and residual nets. However, this modified measure is still only associated with generalization, is highly sensitive to experimental details (e.g. size of models being compared), and does not provide any direct intuition for model comparison given its reliance on parameters alone.

In the PAC-Bayesian measure used in \citet{jiang2019fantastic} we consider flatness around the solution in arbitrary directions. In Eq.~\ref{eqn: pacbayes-expec} we are considering an expectation over random perturbations, ultimately measuring the size of the region of low loss around an optimum in the loss surface. We see that this PAC-Bayesian measure actually increases with model size and is \emph{anti-correlated} with generalization performance.  For a fixed value of $\sigma$ the average magnitude of the perturbation, $u$, will increase as the number of parameters in the model increases. To combat the growth in magnitude of the perturbation, $\sigma$ must become smaller as the number of parameters in the model (i.e. the dimensionality of $u$) grows. The necessary decay in $\sigma$ as models grow leads to the measure in Eq.~\ref{eqn: pacbayes} becoming larger, as we see in practice. For the models that overfit the training data between widths $8$ and $12$, the parameters are highly sensitive to perturbations, and PAC-Bayesian bounds increase by $3$ to $4$ orders of magnitude. Therefore, for clarity, we exclude the PAC-Bayesian bounds for models under $15$ units wide, as the presence of a small number of ill-behaved points prevents the overall structure of the bounds for larger models from being visible. This effect is still pronounced in the experiment with label noise in Figure~\ref{fig: ddlabel}, making the significant increase in the measure with model size less visible (seen also in Figure~\ref{fig: pacbayes_cifar}). We additionally compute the magnitude-aware PAC-Bayes measure which helps mitigate these issues, though they still persist to an extent.

Not all flatness-based generalization measures are equivalent. In contrast to the PAC-Bayesian flatness measure considered in \citet{jiang2019fantastic}, effective dimension essentially computes the number of directions in parameter space that are flat, as determined by the curvature of the loss surface. Combined with our observations about function-space homogeneity in Section~\ref{sec: loss-surfaces}, we see that effective dimension provides a proxy for model compression, which is in informative about generalization performance, in addition to providing correlation with performance. By contrast, the PAC-Bayes flatness measure considers the size of the basin surrounding an optimum, and is highly sensitive to the sharpest direction. Accordingly, we find that effective dimension tends to be significantly more robust to experimental details than both the PAC-Bayes flatness measure and path-norm.

\section{Conclusion}

We have shown how effective dimensionality can be used to gain insight into a range of phenomena, including double descent, posterior contraction, loss surface structure, and function-space diversity of models. As we have seen, simple parameter counting can be a misleading proxy for model complexity and generalization performance; models with many parameters combined with a particular functional form can give rise to simple explanations for data. Indeed, we have seen how depth and width have different effects on generalization performance, regardless of the total number of parameters in the model. In all cases, effective dimensionality tracks generalization performance for models with comparable training loss, helping to explain behaviour that appears mysterious when measured against simple parameter counting. Moving forward, we hope our work will help inspire a continued effort to capture the nuanced interplay between the statistical properties of parameter space and function space in understanding generalization behaviour.

\subsection*{Acknowledgements}

WJM, GB, and AGW were supported 
by an Amazon Research Award, Facebook Research, NSF I-DISRE 193471, NIH R01 DA048764-01A1, NSF IIS-1563887, and NSF IIS-1910266.
WJM was additionally supported by an NSF Graduate Research Fellowship under Grant No. DGE-1839302.
We would like to thank Jayson Salkey and Pavel Izmailov for helpful discussions and Coffee Project NY for their deconstructed lattes.


\bibliographystyle{icml2020}
\bibliography{refs}

\cleardoublepage
\appendix
\renewcommand\thefigure{A.\arabic{figure}}   
\setcounter{figure}{0}

\section{The Hessian and Effective Dimensionality over the Course of Training}
\label{app:hess_training}

One possible limitation of using the Hessian as a measurement for posterior contraction for (Bayesian) deep learning would be if the Hessian was constant through the training procedure, or if the eigenvalues of the Hessian remained constant.
\citet{jacot_neural_2018} showed that in the limit of infinite width neural networks, the Hessian matrix converges to a constant,
in a similar manner to how the Fisher information matrix and Jacobian matrices converge to a constant limit, producing the neural tangent kernel (NTK) \citep{jacot_neural_2018}.
However, \citet{lee2019wide} recently showed that while the infinite width NTK is a good descriptor of finitely wide neural networks, the corresponding finite width NTK is not constant throughout training.
Similarly, the empirical observations of \citet{papyan2018full}, \citet{sagun2016eigenvalues}, and \citet{ghorbani2019investigation} demonstrate that even for extremely wide neural networks, the Hessian is not constant through training.

Preliminary experiments with both the Fisher information matrix (using fast Fisher vector products as described in \citet{maddox_linearizing_2019}) and the NTK demonstrated similar empirical results in terms of double descent and effective dimensionality as the Hessian matrix.

\section{Further Statements on Effective Dimensionality}
\label{app:effective_dimensionality_extensions}
In this section, we provide further results the effective dimensionality, including its connection to both the bias-variance decomposition of predictive risk \citep{geman1992neural,dobriban2018} as well as the Hilbert space norm of the induced kernel \citep{rasmussen_gaussian_2008}.

\subsection{Effective Dimensionality of the Inverse of $A$}\label{app:neff_proof}
We show that 
\begin{align}
rank(A) - N_{eff}(A, \alpha) = N_{eff}(A^{+}, 1/\alpha)
\label{eq:neff-tradeoff},
\end{align}
formalizing the idea that as the effective dimensionality of the covariance increases, the effective dimensionality of the inverse covariance decreases. This statement is alluded to in the analysis of \citet{mackay1992interpolation} but is not explicitly shown.

We assume that $A$ has rank $r$ and that $\alpha \neq 0;$ we also assume that $A^+$ is formed by inverting the non-zero eigenvalues of $A$ and leaving the zero eigenvalues fixed in the eigendecomposition of $A$ (i.e. the Moore-Penrose pseudo-inverse).
With $\lambda_i$ as the eigenvalues of $A$, we can see that
\begin{align*}
r - N_{eff}(A, \alpha) &= \sum_{i=1}^r \frac{\lambda_i + \alpha - \lambda_i}{\lambda_i + \alpha} = \alpha \sum_{i=1}^r \frac{1}{\lambda_i + \alpha}  \\
&=\sum_{i=1}^r \frac{1}{1/\alpha}\frac{1}{\lambda_i + \alpha}
=\sum_{i=1}^r \frac{1}{\lambda_i / \alpha + 1}  \\
&= \sum_{i=1}^r \frac{1/\lambda_i}{1/\lambda_i(\lambda_i / \alpha + 1)} = \sum_{i=1}^r \frac{1/\lambda_i}{1/\alpha + 1/\lambda_i} \\
&= N_{eff}(A^+, 1/\alpha).
\end{align*}
When $A$ is invertible, the result reduces to $k - N_{eff}(A, \alpha) = N_{eff}(A^{-1}, 1/\alpha)$ for $A \in \mathbb{R}^{k \times k}$.

\subsection{Predictive Risk for Bayesian Linear Models}\label{app:pred_risk}
\citet{dobriban2018} and \citet{hastie_surprises_2019} have extensively studied over-parameterized ridge regression. In particular, Theorem 2.1 of \citet{dobriban2018} gives the predictive risk (e.g. the bias-variance decomposition of \citet{geman1992neural}) as a function of effective dimensionality and intrinsic noise.
The critical aspect of their proof is to decompose the variance of the estimate into the effective dimensionality and a second term which then cancels with the limiting bias estimate.
For completeness, we restate Theorem 2.1 of \citet{dobriban2018} theorem for fixed feature matrices, $\Phi,$ and an explicit prior on the parameters, $\beta \sim \mathcal{N}(0, \alpha^2 I)$, leaving the proof to the original work.
\begin{theorem}[Predictive Risk of Predictive Mean for Ridge Regression]
	Under the assumption of model correct specification, $y = \Phi \beta + \epsilon,$ with $\beta$ drawn from the prior and $\epsilon \sim \mathcal{N}(0, I_n),$
	and defining $\hat{f} = \Phi \hat{\beta},$ with $\hat{\beta} = (\Phi^\top \Phi + \alpha^{-2} I)^{-1} \Phi^\top y$ (the predictive mean under the prior specification), then
	\begin{align}
	R(\Phi) = \mathbb{E}&(||Y - \hat{f}||_2^2) = 1 +  \frac{1}{n} N_{eff}(\Phi\Phi^\top, \alpha^{-2}).
	\end{align}
\end{theorem}

\subsection{Expected RKHS Norm}\label{app:ehsn}
Finally, we show another unexpected connection of the effective dimensionality --- that the reproducing kernel Hilbert space (RKHS) norm is in expectation, under model correct specification, the effective dimensionality.
We follow the definition of Gaussian processes of \citet{rasmussen_gaussian_2008} and focus on the definition of the RKHS given in \citet[Chapter 6]{rasmussen_gaussian_2008}, which is defined as $||f||_\mathcal{H}^2 = \langle f,f\rangle_\mathcal{H}=\sum_{i=1}^N f_i^2/\lambda_i,$ where $\lambda_i$ are the eigenvalues associated with the kernel operator, $K,$ of the RKHS, $\mathcal{H}$.\footnote{Note that the expectation we take in the following is somewhat separate than the expectation taken in \citet{rasmussen_gaussian_2008} which is directly over $f_i.$}
The kernel is the covariance matrix of the Gaussian process, and assuming that the response is drawn from the same model, then $y \sim \mathcal{N}(0, K + \sigma^2 I),$ then $a= (K + \sigma^2 I)^{-1} y,$ where $a$ is the optimal weights of the function with respect to the kernel, e.g. $f = \sum_{i=1}^N a_i K(x, .).$
To compute the Hilbert space norm, we only need to compute the optimal weights and the eigenvalues of the operator. For finite (degenerate) Hilbert spaces this computation is straightforward:
\begin{align*}
\mathbb{E}_{p(y)}&(||f||_\mathcal{H}^2) =\mathbb{E}_{p(y)}(a^\top K a) \\
&= \mathbb{E}_{p(y)}(y^\top (K+\sigma^2 I)^{-1} K (K+\sigma^2 I)^{-1} y) \\
&=\mathbb{E}_{p(y)} tr(y^\top (K+\sigma^2 I)^{-1} K (K+ \sigma^2 I)^{-1} y) \\
&= \mathbb{E}_{p(y)} tr((K+\sigma^2 I)^{-1} K (K+ \sigma^2 I)^{-1} y y^\top) \\
&= tr((K+\sigma^2 I)^{-1} K (K+ \sigma^2 I)^{-1} (K+\sigma^2I)) \\
& = N_{eff}(K, \sigma^2)
\end{align*}
with the second equality coming by plugging in the optimal $a$ (see \citet[][Chapter 6]{rasmussen_gaussian_2008} and \citet{belkin2019reconciling} as an example).
As linear models with Gaussian priors are Gaussian processes with a degenerate feature expansion, the expected RKHS norm becomes $N_{eff}(\Phi^\top \Phi, \sigma^2/\alpha^2),$ which is the same value as our definition of posterior contraction. 
Further research connecting these two ideas is needed.

\section{Measuring Posterior Contraction in Bayesian Generalized Linear Models}
\label{app:linear_eff}

We first consider the over-parametrized case, $k > n$:

\begin{align}
    \Delta_{post}(\theta) &= tr(Cov_{p(\theta)}(\theta)) - tr(Cov_{p(\theta|\mathcal{D})}(\theta)) \nonumber\\
    &=\sum_{i=1}^k \alpha^2 - \sum_{i=1}^n (\lambda_i + \alpha^{-2})^{-1} + \sum_{i=n+1}^k \alpha^2 \nonumber \\
    &=k\alpha^2 - (k - n)\alpha^2 - \sum_{i=1}^n (\lambda_i + \alpha^{-2})^{-1} \nonumber \\
    &=\sum_{i=1}^n\frac{1 - \alpha^2 (\lambda_i + \alpha^{-2})}{\lambda_i + \alpha^{-2}} \nonumber \\
    &=\alpha^2 \sum_{i=1}^n \frac{\lambda_i}{\lambda_i + \alpha^{-2}};
\end{align}
where we have used Theorem \ref{thm: post-contraction} to assess the eigenvalues of the posterior covariance.
When $n > k$, we have the simplified setting where the summation becomes to $k$ instead of $n$, giving us that all of the eigenvalues are shifted from their original values to become $\lambda_i + \alpha^{-2},$ and so
\begin{align}
    \Delta_{post.}(\theta) = \alpha^{-2} \sum_{i=1}^k \frac{\lambda_i}{\lambda_i + \alpha^{-2}},
\end{align}
where $\lambda_i$ is the $i$th eigenvalue of $\Phi^\top \Phi/\sigma^2.$

\subsection{Contraction in Function Space}
We can additionally consider the posterior contraction in function space.
For linear models, the posterior covariance on the training data in function space becomes
\begin{align}
    \Phi \Sigma_{\beta | \mathcal{D}}\Phi^\top = \sigma^2 \Phi (\Phi^\top \Phi + \frac{\sigma^2}{\alpha^2}I_p)^{-1} \Phi^\top,
\end{align}
while the prior covariance in function space is given by $\alpha^2 \Phi \Phi^\top.$
We will make the simplifying assumption that the features are normalized such that $tr(\Phi \Phi^\top) = rank(\Phi \Phi^\top) = r.$
Now, we can simplify
\begin{align*}
    \Delta_{post}(f) &= tr(Cov_{p(f)}(f) - tr(Cov_{p(f|\mathcal{D})}(f)) \nonumber \\
    &=\alpha^2 r - \sigma^2 \sum_{i=1}^r \frac{\lambda_i}{\lambda_i + \sigma^2 / \alpha^2} \nonumber \\
   &=\alpha^2 \sum_{i=1}^r \frac{\lambda_i + \sigma^2 / \alpha^2}{\lambda_i + \sigma^2 / \alpha^2} - \sigma^2 \sum_{i=1}^r \frac{\lambda_i}{\lambda_i + \sigma^2 / \alpha^2} \nonumber \\
   =(\alpha^2& - \sigma^2)\sum_{i=1}^r \frac{\lambda_i}{\lambda_i + \sigma^2/\alpha^2} + \sigma^2 \sum_{i=1}^r \frac{1}{\lambda_i + \sigma^2/\alpha^2} \nonumber \,. \\
\end{align*}
Simplifying and recognizing these summations as the effective dimensionalities of $\Phi^\top \Phi$ and $(\Phi^\top \Phi)^+$,  we get that
\begin{align}
  \Delta_{post}(f) &= (\alpha^2 - \sigma^2)N_{eff}(\Phi^\top \Phi, \sigma^2/\alpha^2) \nonumber \\
  &\hspace{1cm} + \sigma^2 N_{eff}((\Phi^\top \Phi)^+, \alpha^2/\sigma^2) \\
  &=\sigma^2r + (\alpha^2 - 2\sigma^2)N_{eff}(\Phi^\top \Phi, \sigma^2/\alpha^2), \nonumber
\end{align}
thereby showing that the posterior contraction in function space is explicitly tied to the effective dimensionality of the Gram matrix.

\section{Posterior Contraction and Function-Space Homogeneity Proofs and Additional Theorems}

In this section we complete the proofs to Theorems \ref{thm: post-contraction} and \ref{thm: function-homog} and extend the results from linear models to generalized linear models.

\subsection{Proof and Extensions to Theorem \ref{thm: post-contraction}}\label{app: post-contraction}
\begin{theorem*}[Posterior Contraction in Bayesian Linear Models]
Let $\Phi = \Phi(x) \in \mathbb{R}^{n \times k}$ be a feature map of $n$ data observations, $x$, with $n < k$ and assign isotropic prior $\beta \sim \mathcal{N}(0_k, S_0 = \alpha^2I_k)$ for parameters $\beta \in \mathbb{R}^k$. Assuming a model of the form $y \sim \mathcal{N}(\Phi \beta, \sigma^2 I_{n})$ the posterior distribution of $\beta$ has an $p-k$ directional subspace in which the variance is identical to the prior variance.
\end{theorem*}

\begin{proof}
The posterior distribution of $\beta$ in this case is known and given as
\begin{equation}\label{eqn: beta-posterior}
\begin{aligned}
\beta|\mathcal{D} &\sim \mathcal{N}\left((\mu | \mathcal{D}), (\Sigma | \mathcal{D})\right) \\
\mu | \mathcal{D} &= (\Phi^\top \Phi/\sigma^2 + S_0^{-1})^{-1} \Phi^\top y/\sigma^2\\
\Sigma | \mathcal{D} &= (\Phi^\top \Phi/\sigma^2 + S_0^{-1})^{-1}
\end{aligned}
\end{equation}
We want to examine the distribution of the eigenvalues of the posterior variance. Let $\Phi^\top\Phi/\sigma^2 = V \lambda_n V^\top$ be the eigendecomposition with eigenvalues $\Lambda = \textrm{diag}(\gamma_1, \dots, \gamma_n, 0_{n+1}, \dots, 0_k)$;  $k-n$ of the eigenvalues are 0 since the gram matrix $\Phi^\top\Phi$ is at most rank $n$ by construction.

Substitution into the posterior variance of $\beta$ yields,
\begin{equation}
\begin{aligned}
     (\Phi^\top \Phi/\sigma^2 + S_0^{-1})^{-1} &= (V\Lambda V^\top + \alpha^{-2}I_k)^{-1}\\
     &= V(\Lambda + \alpha^{-2}I_k)^{-1}V^\top\\
     &= V\Gamma V^\top.
\end{aligned}
\end{equation}
The eigenvalues of the posterior covariance matrix are given by the entries of $\Gamma$, $\left((\gamma_{1} + \alpha^{-2})^{-1}, \dots, (\gamma_n + \alpha^{-2})^{-1}, \alpha^2, \dots, \alpha^2\right)$, where there are $k-n$ eigenvalues that retain a value of $\alpha^2$.

Therefore the posterior covariance has $p-n$ directions in which the posterior variance is unchanged and $n$ directions in which it has contracted as scaled by the eigenvalues of the gram matrix $\Phi^\top \Phi$.
\end{proof}

Generalized linear models (GLMs) do not necessarily have a closed form posterior distribution.
However, \citet{neal_high_2006}  give a straightforward argument using the invariance of the likelihood of GLMs to orthogonal linear transformation in order to justify the usage of PCA as a feature selection step. We can adapt their result to show that overparameterized GLMs have a $k -n$ dimensional subspace in which the posterior variance is identical to the prior variance.

\begin{theorem}[Posterior Contraction in Generalized Linear Models]\label{thm: post-contraction-glm}
We specify a generalized linear model, $E[Y] = g^{-1}(\Phi \beta)$ and $Var(Y) = V(g^{-1}(\Phi \beta))$, where $\Phi \in \mathbb{R}^{n \times k}$ is a feature matrix of $n$ observations and $k$ features and $\beta \in \mathbb{R}^k$ are the model parameters. In the overparameterized setting with isotropic prior on $\beta$, there exists a $k-n$ dimensional subspace in which the posterior variance is identical to the prior variance.
\end{theorem}
\begin{proof}
First note that the likelihood of a GLM takes as argument $\Phi \beta$, thus transformations that leave $\Phi \beta$ unaffected leave the likelihood, and therefore the posterior distribution, unaffected.

Let $R$ be an orthogonal matrix, $R^\top R = RR^\top = I_p$, and $\tilde{\beta} = R\beta \sim N(0, \sigma^2 I)$.
If we assign a standard isotropic prior, to $\beta$ then $\tilde{\beta} = R\beta \sim \mathcal{N}(0, \sigma^2 R I_k R^\top=\sigma^2I_k)$. If we also rotate the feature matrix, $\tilde{\Phi} = \Phi R^\top \in \mathbb{R}^{n \times k}$ so that $\tilde{\Phi} \tilde{\beta} = \Phi R^\top R \beta = \Phi \beta$, showing that the likelihood and posterior remain unchanged under such transformations.

In the overparameterized regime, $k > n$, with linearly independent features we have that $\Phi$ has rank at most $k$, and we can therefore choose $R$ to be a rotation such that $\Phi R$ has exactly $k-n$ columns that are all 0. This defines a $k-n$ dimensional subspace of $\beta \in \mathbb{R}^k$ in which the the likelihood is unchanged. Therefore the posterior remains no different from the prior distribution in this subspace, or in other words, the posterior distribution has not contracted in $k-n$ dimensions.
\end{proof}

\subsection{Function-Space Homogeneity}\label{app: function-homog}

\begin{theorem*}[Function-Space Homogeneity in Linear Models]
Let $\Phi = \Phi(x) \in \mathbb{R}^{n \times k}$ be a feature map of $n$ data observations, $x$, with $n < k$ and assign isotropic prior $\beta \sim \mathcal{N}(0_k, S_0 = \alpha^2I_k)$ for parameters $\beta \in \mathbb{R}^k$.
The minimal eigenvectors of the Hessian define a $k-n$ dimensional subspace in which parameters can be perturbed without changing the training predictions in function-space.
\end{theorem*}

\begin{proof}
The posterior covariance matrix for the parameters is given by
$$
\Sigma_{\beta | \mathcal{D}} = \left(\frac{\Phi^\top \Phi}{\sigma^2} + \alpha^{-2}I_k\right)^{-1},
$$
and therefore the Hessian of the log-likelihood is $\left(\frac{\Phi^\top \Phi}{\sigma^2} + \alpha^{-2}I_k\right)$. By the result in Theorem \ref{thm: post-contraction}  there are $k-n$ eigenvectors of the Hessian all with eigenvalue $\alpha^{-2}$. If we have some perturbation to the parameter vector $u$ that resides in the span of these eigenvectors we have
$$
\left(\frac{\Phi^\top \Phi}{\sigma^2} + \alpha^{-2}I_k\right)u = \alpha^{-2}u,
$$
which implies $u$ is in the nullspace of $\Phi^\top \Phi$. By the properties of gram matrices we have that the nullspace of $\Phi^\top \Phi$ is the same as that of $\Phi$,
thus $u$ is also in the nullspace of $\Phi$
Therefore any prediction using perturbed parameters takes the form $\hat{y} = \Phi(\beta +u) = \Phi \beta$, meaning the function-space predictions on training data under such perturbations are unchanged.
\end{proof}

\begin{theorem}[Function-Space Homogeneity in Generalized Linear Models]
We specify a generalized linear model, $E[Y] = g^{-1}(\Phi \beta)$, where $\Phi \in \mathbb{R}^{n \times k}$ is a feature matrix of $n$ observations and $k$ features and $\beta \in \mathbb{R}^k$ are the model parameters. In the overparameterized setting with isotropic prior on $\beta$, there exists a $k-n$ dimensional subspace in which parameters can be perturbed without changing the training predictions in function-space or the value of the Hessian.
\end{theorem}

\begin{proof}
The Hessian of the log-likelihood for GLMs can be written as a function of the feature map, $\Phi$, and the product of the feature map and the parameters, $\Phi \beta$, i.e. $\beta$ only appears multiplied by the feature map \citep{nelder1972generalized}. We can then write $\mathcal{H}_\beta = f(\Phi\beta, \Phi)$ Additionally predictions are generated by $y = g^{-1}\left(\Phi \beta\right)$. Since $\Phi \in \mathbb{R}^{n\times p}$ with $n < p$ there is a nullspace of $\Phi$ with dimension at least $n-p$.
Thus for any $u \in \textrm{null}(\Phi)$ we have $g^{-1}\left(\Phi(\beta + u)\right) = g^{-1}\left(\Phi \beta\right) = y$ and $f(\Phi(\beta + u), \Phi) = f(\Phi\beta, \Phi) = \mathcal{H}_\beta$, which shows that the training predictions and the Hessian remain unchanged.
\end{proof}

\section{Perturbations on CIFAR-$10$}
\label{app: cifar}

To demonstrate that the results presented in Section \ref{sec: loss-surfaces} apply to larger architectures similar to those seen in practice we train a convolutional classifier provided by Pytorch
on the CIFAR-$10$ dataset.\footnote{The architecture is provided here: \url{https://pytorch.org/tutorials/beginner/blitz/cifar10_tutorial.html}} The network has approximately $62000$ parameters and is trained on $50000$ images.

Figure \ref{fig: cifar-loss} shows the presence of degenerate directions in parameter space. We compute the top $200$ eigenvectors of the Hessian of the loss and consider perturbations in the directions of the top $2$ eigenvectors, as well as in all parameter directions \emph{except} the top $200$ eigenvectors of the Hessian. We see that even for larger networks and more complex datasets degenerate directions in parameter space are still present and comprise most possible directions.

\begin{figure}
    \centering
    \includegraphics[width=\linewidth]{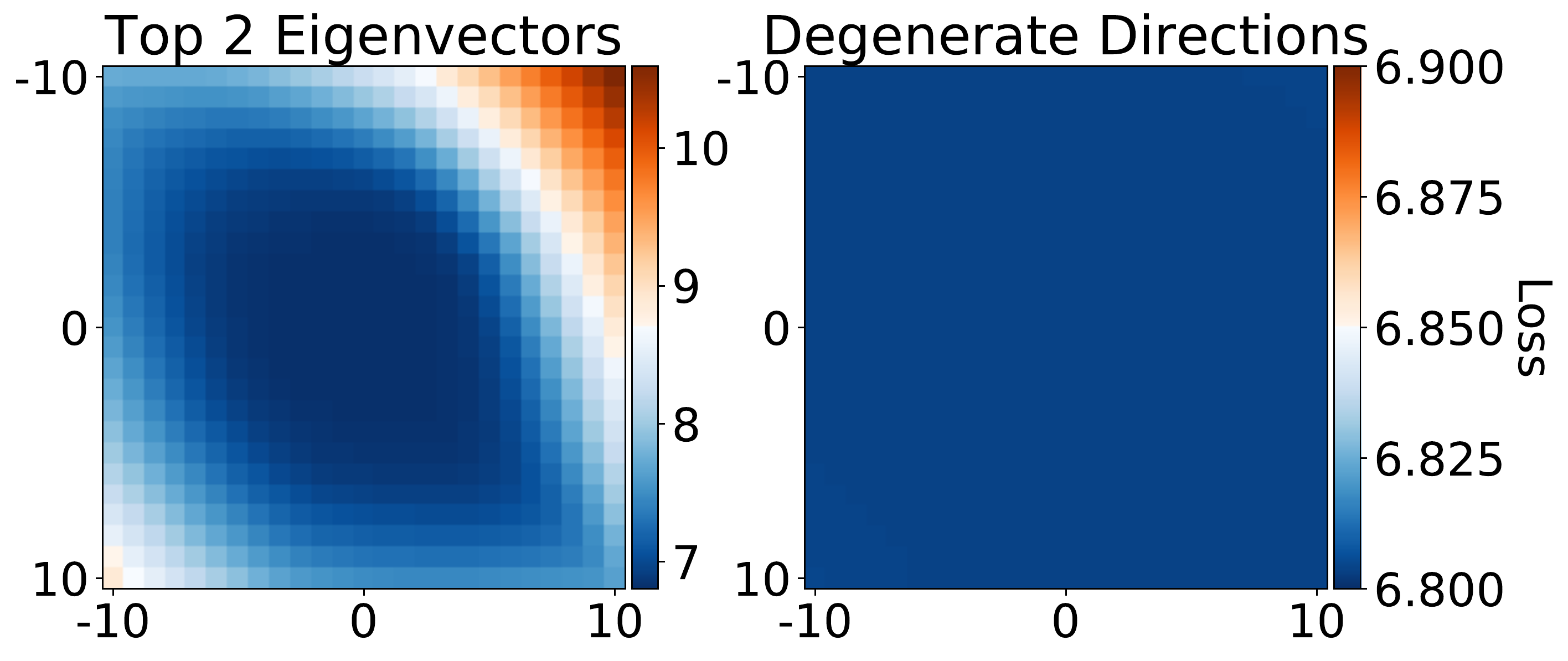}
    \caption{\textbf{Left:} A visualization of the log-loss surface taken in the direction of the top two eigenvectors of the Hessian of the loss. \textbf{Right:} A visualization of a random projection of the log-loss surface in all parameter directions \emph{except} the top $200$ eigenvectors of the Hessian. We can see that in nearly all directions the loss is constant even as we move far from the optimal parameters. \textbf{Note} the scale difference, even as we increase the resolution of the degenerate loss surface we still see no structure.}
    \label{fig: cifar-loss}
\end{figure}

Figure \ref{fig: cifar-homog} demonstrates that the degenerate directions in parameter space lead to models that are homogeneous in function space on both training and testing data. As increasingly large perturbations are made in degenerate parameter directions, we still classify more than $99\%$ of both training and testing points the same as the unperturbed classifier.

\begin{figure}
    \centering
    \includegraphics[width=\linewidth]{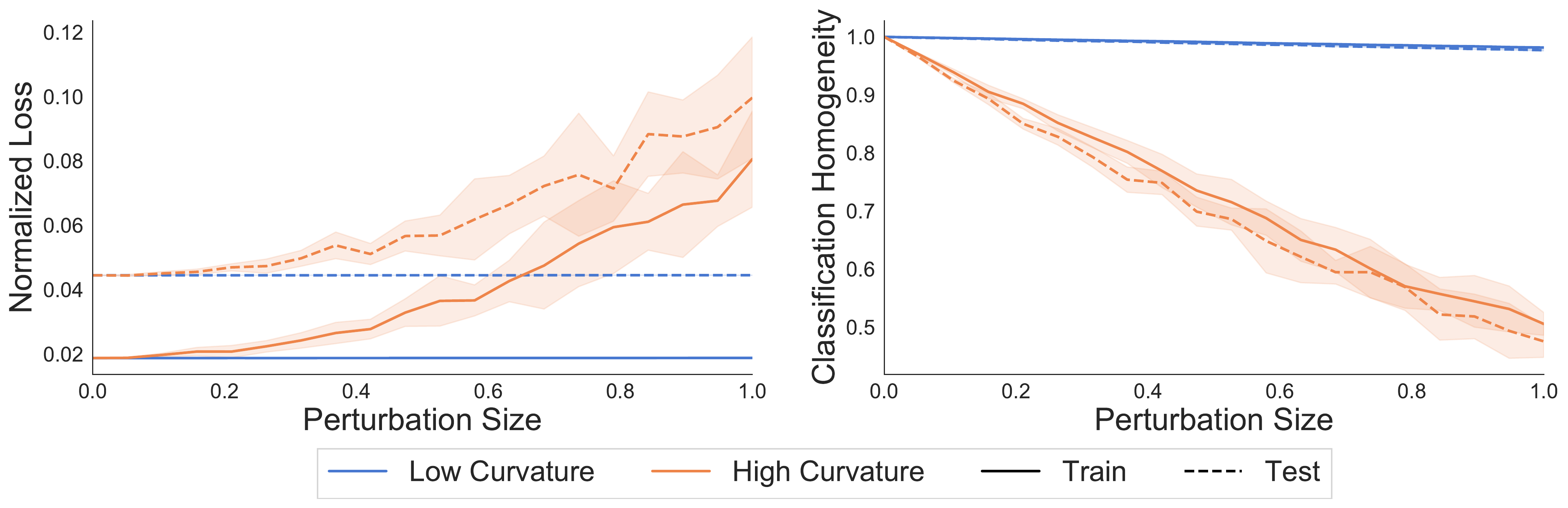}
    \caption{\textbf{Left:} Loss, normalized by dataset size, on both train and test sets as perturbations are made in high curvature directions and degenerate directions. \textbf{Right:} Classification homogeneity, the fraction of data points classified the same as the unperturbed model, as perturbations are made in both high curvature and degenerate directions.}
    \label{fig: cifar-homog}
\end{figure}

\section{More Classifiers}
\label{app: classifiers}
Figures \ref{fig: low-curve-classifiers} and \ref{fig: high-curve-classifiers} provide more examples of perturbations in high and low curvature directions and the effect of the scale of the perturbation on function-space predictions for the two-spirals experiment in Section \ref{sec: loss-surfaces}.
\begin{figure}[H]
    \centering
    \includegraphics[width=\linewidth]{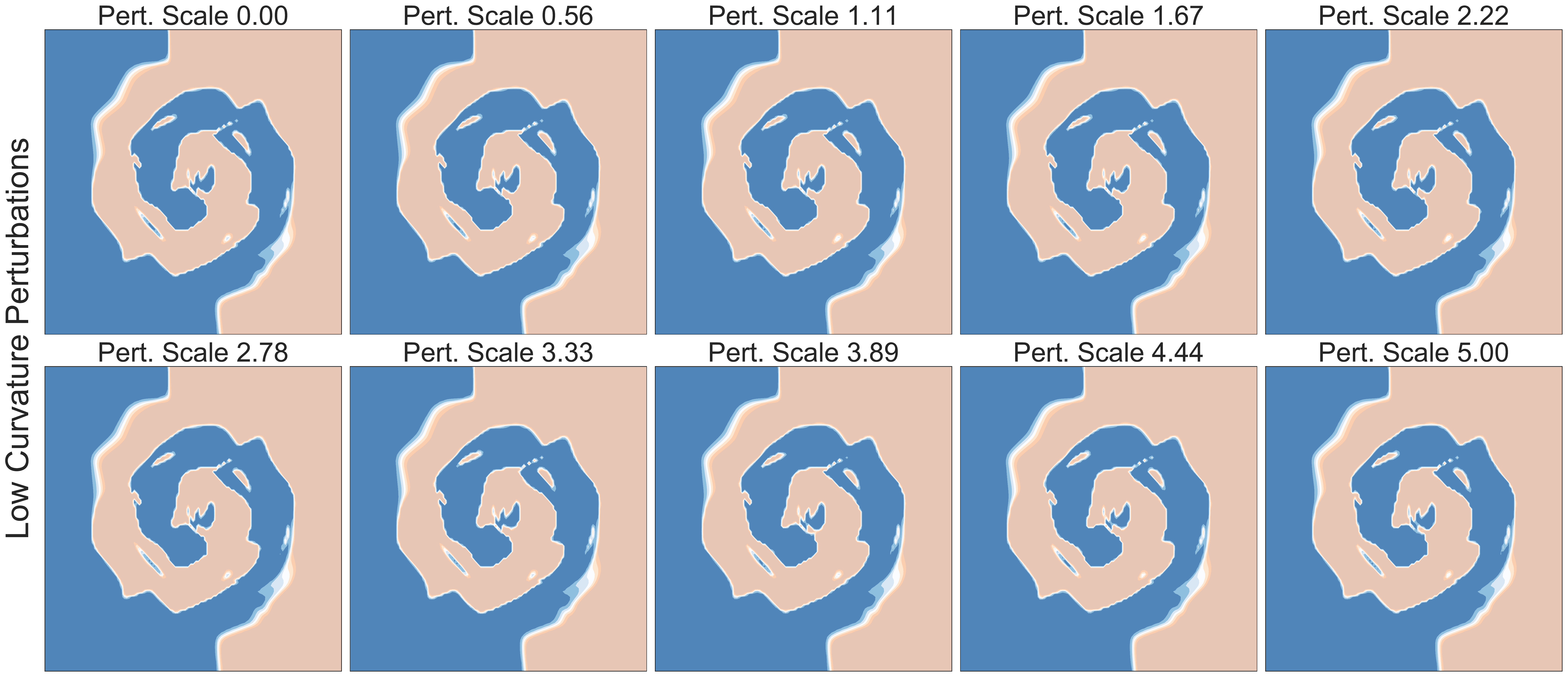}
    \caption{Classifiers as the parameters are shifted in random directions within the span of the bottom 1500 eigenvectors of the Hessian of the loss. Scales of the perturbation range from 0 (upper left) to 2 (lower right).}
    \label{fig: low-curve-classifiers}
\end{figure}

\begin{figure}[H]
    \centering
    \includegraphics[width=\linewidth]{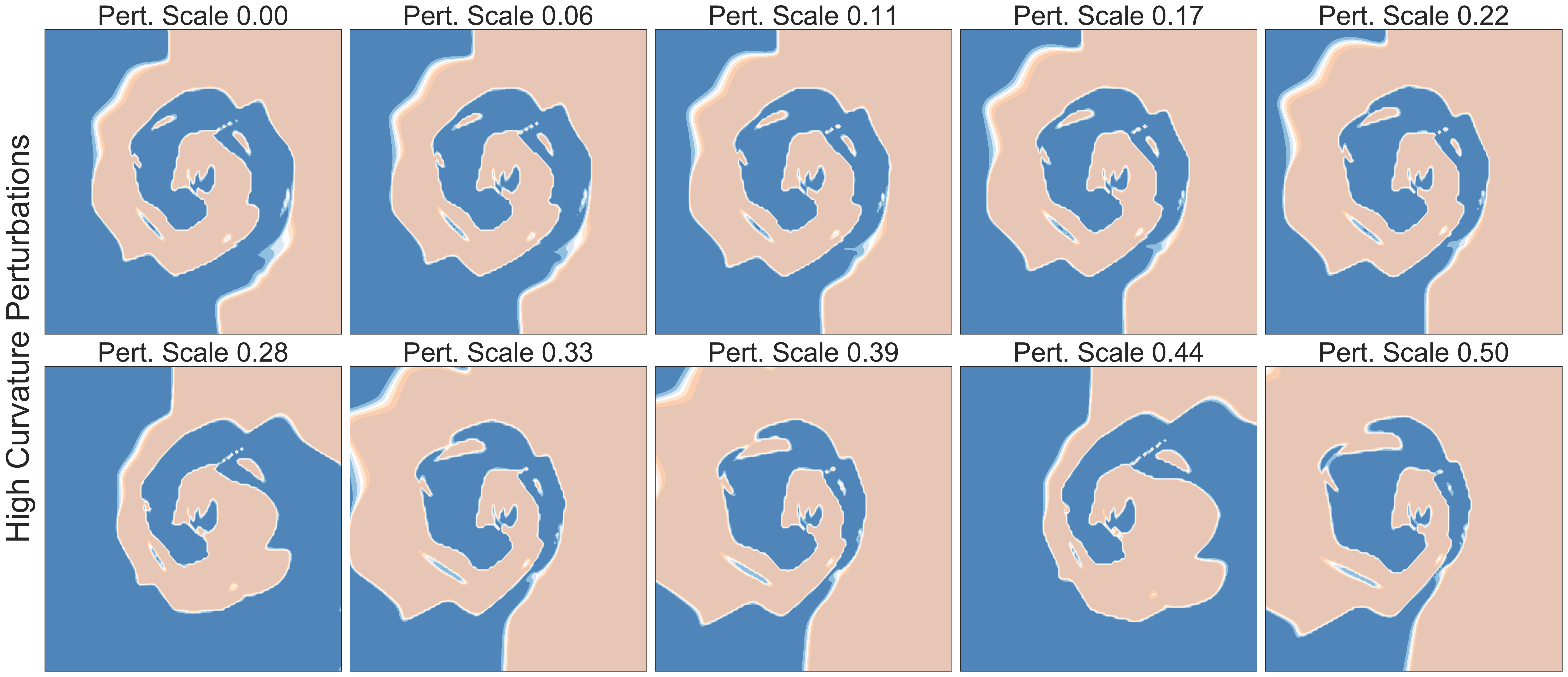}
    \caption{Classifiers as the parameters are shifted in random directions within the span of the top 3 eigenvectors of the Hessian of the loss. Scales of the perturbation range from 0 (upper left) to 0.5 (lower right).}
    \label{fig: high-curve-classifiers}
\end{figure}

\section{Deep Neural Network Training and Eigenvalue Computation}\label{app:dnn_exp_details}

\paragraph{Training Details}
For the double descent experiments in Figures \ref{fig:nn-double-descent-intro} and \ref{fig:width_depth_exp} we use neural network architectures from the following sources:
\begin{itemize}
    \item CNNs from \url{https://gitlab.com/harvard-machine-learning/double-descent/-/blob/master/models/mcnn.py} but also include an option to vary the depth,
    \item ResNet18 from \url{https://gitlab.com/harvard-machine-learning/double-descent/-/blob/master/models/resnet18k.py},
    \item PreResNet-56 from \url{https://github.com/bearpaw/pytorch-classification/blob/master/models/cifar/preresnet.py}.
\end{itemize}

Specifically, we train with SGD with a learning rate of $10^{-2}$, momentum of $0.9,$ weight decay of $10^{-4}$ (thus corresponding approximately to a Gaussian prior of with variance $1000$) for $200$ epochs with a batch size of $128.$
The learning rate decays to $10^{-4}$ following the piecewise constant learning rate schedule in \citet{izmailov2018averaging} and \citet{maddox2019simple}, beginning to decay at epoch $100$.
We use random cropping and flipping for data augmentation --- turning off augmentations to compute eigenvalues of the Hessian.

\paragraph{Lanczos Calculations}
We use GPU enabled Lanczos as implemented in \citet{gardner2018gpytorch} to compute the eigenvalues approximately, running $100$ steps to compute approximately $100$ of the top eigenvalues. 
We note that our estimates of the effective dimensionality are somewhat biased from not including all of the small eigenvalues.
However, these small eigenvalues will contribute negligibly and Lanczos will converge to the true eigenvalues if we ran $k$ steps where $k$ is the rank of the Hessian.

\section{Double Descent and Effective Dimensionality: Further Experiments}
\begin{figure*}[!t]
\centering
\includegraphics[width=\linewidth]{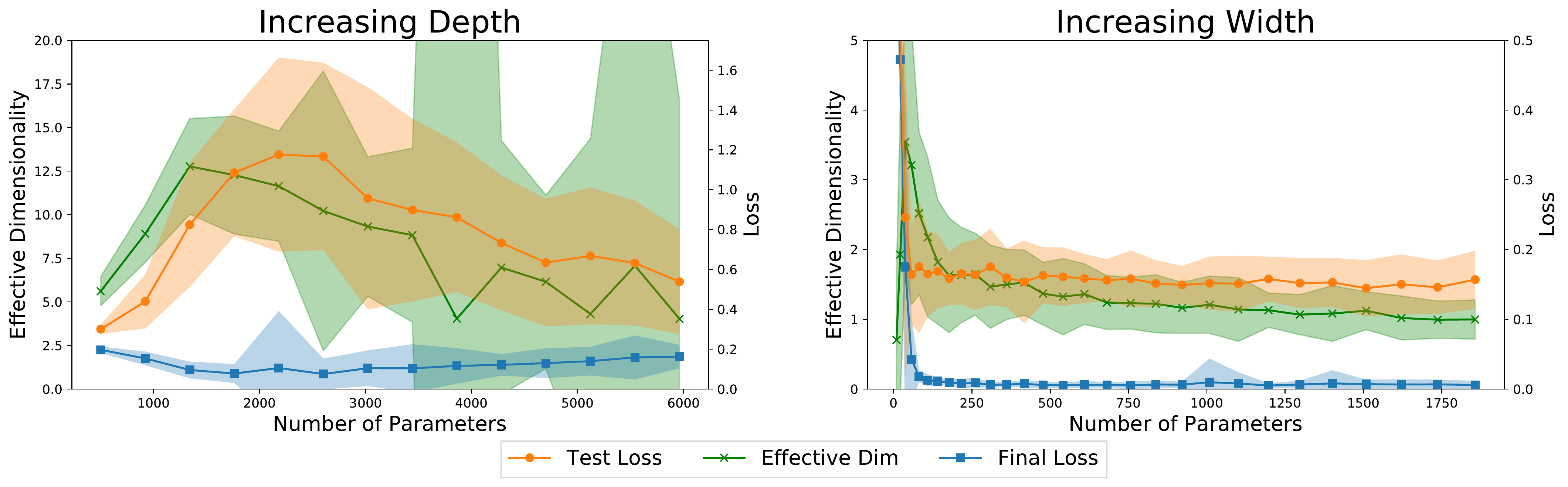}
\caption{\textbf{Left:} Increasing depth on the two spirals problem. Clearly seen is a double descent curve where the test loss first increases before decreasing as a function of depth. The effective dimensionality follows the same trend. \textbf{Right:} Increasing width on the two spirals problem. Here, increased width produces constant test performance after the training loss reaches zero, and the effective dimensionality stays mostly constant. Shading represents two standard deviations as calculated by $25$ random generations of the spirals data.}
\label{fig:width_depth_exp_two_spirals}
\end{figure*}

Finally, we consider several further experiments on the two spirals problem to test the effects of increasing depth and width to serve as a sanity check for our results on both ResNets and CNNs. 
In Figure \ref{fig:width_depth_exp_two_spirals}, we fix the number of data points to be $3000$, and vary the depth of the neural network (20 hidden units at each layer, ELU activations) using between one and $15$ hidden units, training for $4000$ steps as before.
Here, we run each experiment with $25$ repetitions and compute all of the eigenvalues of the Hessian at convergence (the largest model contains $6000$ parameters).
In the left panel, we see a pronounced double descent curve with respect to both effective dimensionality and test loss as we vary depth.
In the right panel, we use the same data points, but use three hidden layer networks, varying the width of each layer between one and $30$ units per layer.
Here, we see only a monotonic decrease in both test error and effective dimensionality with increasing width not helping that much in terms of test error --- the effective dimensionality is highest for the models with smallest size and slowly decreases as the width is increased.
These results serve as a sanity check on our large-scale Lanczos results in the main text.

In Figure \ref{fig:gen_vs_eff_dim}, we plot the effective dimensionality against the test error for the linear model example in Section \ref{sec: double-descent}.
A clear linear-looking trend is observed, which corresponds to the models that have nearly zero training error. 
The bend near the origin is explained by models that do not have enough capacity to fit --- therefore, their effective dimensionality is very small.
We observe a similar trend for ResNets and CNNs.
\begin{figure*}[h]
    \centering
    \includegraphics[width=0.5\linewidth]{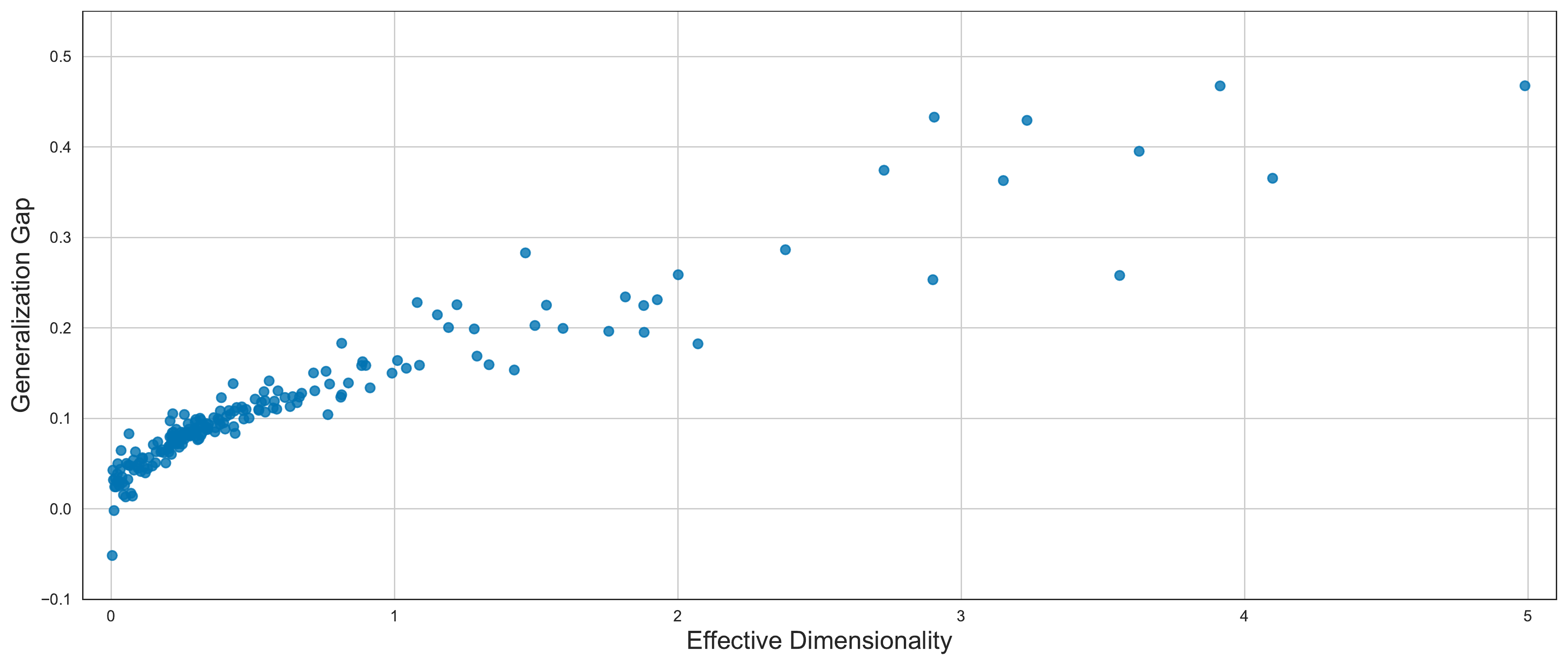}
    \caption{Effective dimensionality plotted against generalization gap (test error $-$ train error) for the linear model of Section \ref{sec: double-descent}. Note that all but the very smallest models with effective dimensionality track nearly linearly with generalization error.}
    \label{fig:gen_vs_eff_dim}
\end{figure*}

In Figure \ref{fig:dd_test_error}, we plot the effective dimensionality against the negative log likelihood for the ResNet18s trained on CIFAR-100 with $20\%$ data corruption. 
We show the sample Pearson correlation with respect to test loss, error, and the generalization gap for the same dataset of effective dimensionality along with PAC-Bayes, magnitude aware PAC-Bayes, path norm, and the logarithm of the path norm in Table \ref{tab:gen-correlation_corruption}.

Finally, in Figure \ref{fig: pacbayes_cifar}, we plot the PAC-Bayes and magnitude aware PAC-Bayes bounds on their raw scale for ResNet18s trained on CIFAR-100 and on the $20\%$ corrupted CIFAR-100 dataset. 
The trends are similar for the bounds across the two datasets, with the PAC-Bayes bound generally increasing as width increases. 
The magnitude aware PAC-Bayes bounds also act similarly, decreasing as width increases.

\begin{table}
\centering
\begin{tabular}{ c|c|c|c|} 
  & Test loss & Test Error & Gen. Gap\\ 
  \hline
  $N_{eff}$(Hessian) & $0.8772$ & $0.87608$ &  $0.8578$\\
  \hline
   PAC-Bayes & $0.6644$ & $0.6424$ & $0.7758$\\ 
 \hline
 Mag. PAC-Bayes & $0.6933$ & $0.6728$ & $0.7956$ \\ 
 \hline
 Path-Norm & $0.6722$ & $0.6478$ & $0.7937$ \\ 
 \hline
 Log Path-Norm & $\mathbf{0.9585}$ & $\mathbf{0.9492}$ & $\mathbf{0.9836}$ \\ 
 \hline

\end{tabular}
\caption{Sample Pearson correlation between generalization measures and the test loss, test error, and generalization gap for ResNet$18$s of varying width trained on CIFAR-$100$ with $20\%$ corruption that achieve a training loss below $0.1$. Among these models effective dimensionality and path-norm are most correlated with test loss, test error, and the generalization gap.}
\label{tab:gen-correlation_corruption}
\end{table}

\begin{figure*}
    \centering
    \includegraphics[width=0.8\columnwidth]{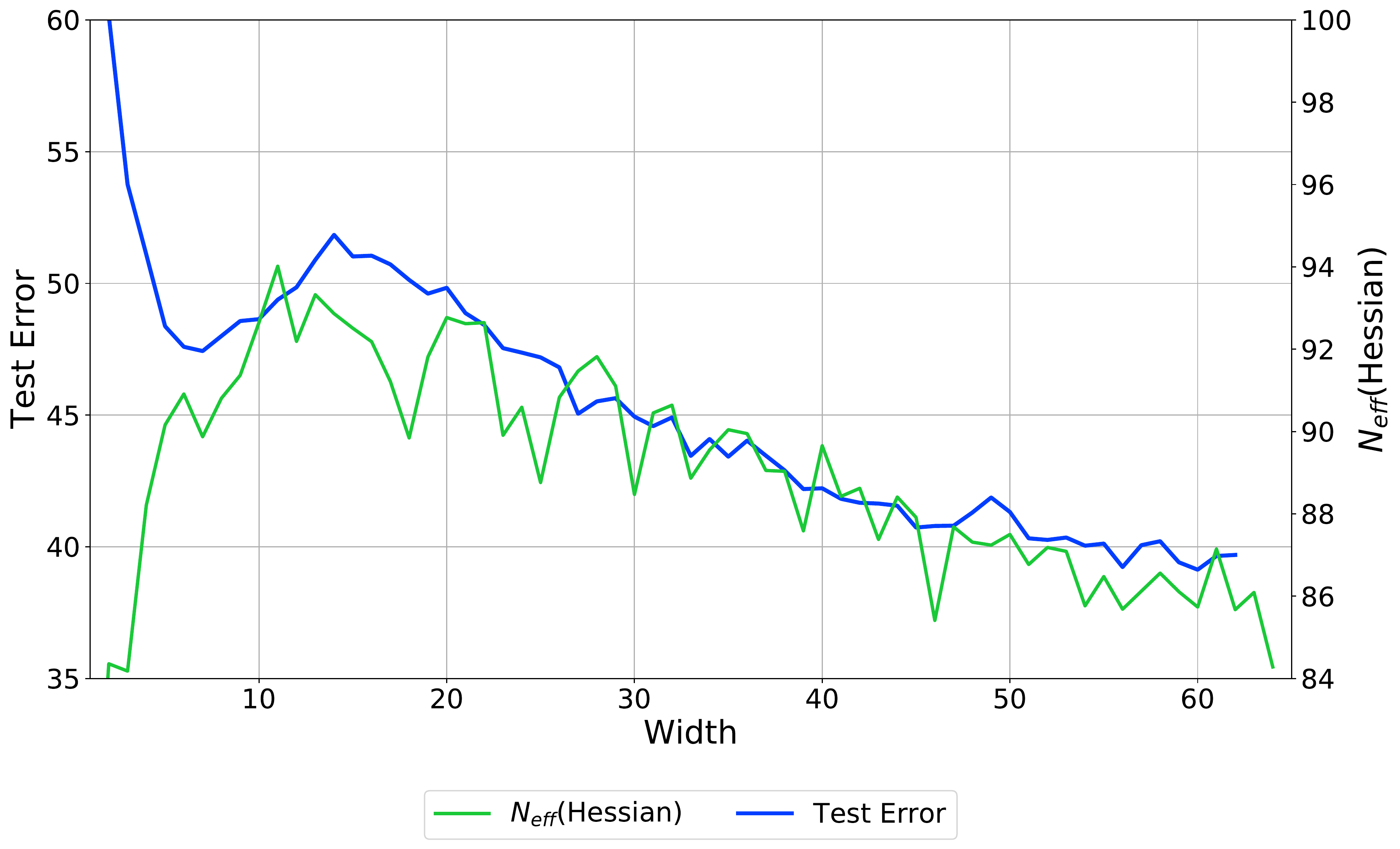}
    \caption{Double descent with respect to test error as demonstrated on ResNet18 on CIFAR-100 with $20\%$ corruption. The effective dimensionality again tracks the double descent curve, this time present in the test \emph{error} rather than test \emph{loss}.}
    \label{fig:dd_test_error}
\end{figure*}

\begin{figure*}
    \centering
    \includegraphics[width=0.8\linewidth]{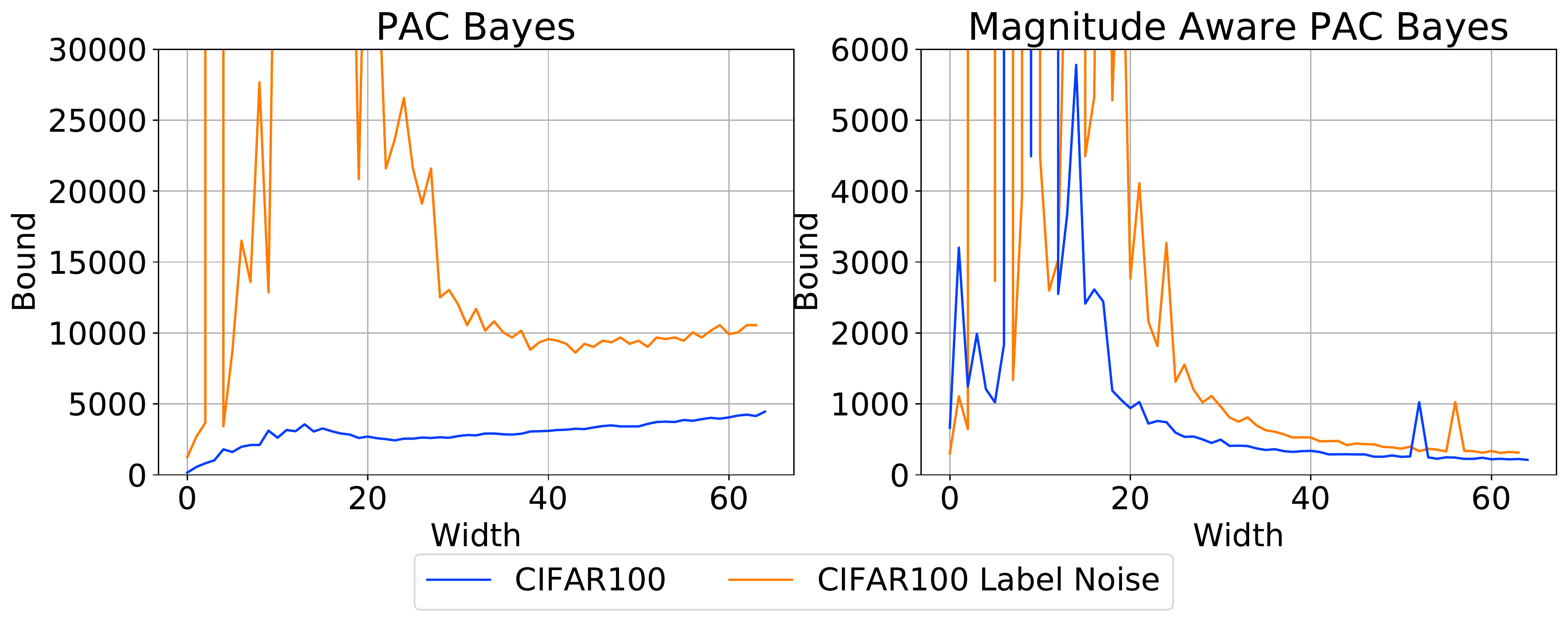}
    \caption{PAC-Bayes and magnitude aware PAC-Bayes bounds displayed on their raw scale for both CIFAR-100 and the $20\%$ corrupted CIFAR-100. All have a significant spike near a width of $20$ which is when training loss first reaches zero. The standard PAC-Bayes bounds increase as width increases, while the magnitude aware ones seem to decrease towards zero.}
    \label{fig: pacbayes_cifar}
\end{figure*}

\end{document}